\documentclass{article}

\usepackage[nonatbib,preprint]{neurips_2026}

\usepackage[utf8]{inputenc} 
\usepackage[T1]{fontenc}    
\usepackage{url}            
\usepackage{booktabs}       
\usepackage{amsmath,amsfonts}       
\usepackage{nicefrac}       
\usepackage{microtype}      
\usepackage{xcolor}         
\usepackage{multirow}
\usepackage{graphicx}
\usepackage{subfig}
\usepackage{wrapfig}
\usepackage[linesnumbered, ruled]{algorithm2e}
\usepackage[colorlinks,
            linkcolor=blue,
            anchorcolor=blue,
            citecolor=blue]{hyperref}
\usepackage{amssymb, amsthm}
\usepackage{bm}
\usepackage{mathtools}

\newtheorem{assumption}{Assumption}
\newtheorem{lemma}{Lemma}
\newtheorem{theorem}{Theorem}
            
\newcommand{\projectname}{CATA}
\newcommand{\partitle}[1]{\noindent \textbf{#1.}}

\title{CATA: Continual Machine Unlearning via Conflict-Averse Task Arithmetic}

\author{
Shen Lin$^{1}$\thanks{These authors contribute equally.}~~, Junhao Dong$^{2}$\footnotemark[1]~~, Rongjie Chen$^{1}$, Xiaoyu Zhang$^{3}$, Li Xu$^{1}$\thanks{Corresponding author.}~~, and Xiaofeng Chen$^{3}$\\ [2mm]
$^{1}$Fujian Normal University, 
$^{2}$Nanyang Technological University,
$^{3}$Xidian University
}

\begin{document}

\maketitle

\begin{abstract}
Vision-language models (VLMs) have shown remarkable ability in aligning visual and textual representations, enabling a wide range of multimodal applications. However, their large-scale training data inevitably raises concerns about privacy, copyright, and undesirable content, creating a strong need for machine unlearning. While existing studies mainly focus on single-shot unlearning, practical VLM deployment often involves sequential removal requests over time, giving rise to continual machine unlearning. In this work, we make the first attempt to study continual unlearning for VLMs and identify three key challenges in this setting: effectiveness in removing target knowledge, fidelity in preserving retained model utility, and persistence in preventing knowledge re-emergence under sequential updates. To address these challenges, we propose \projectname, a conflict-averse task arithmetic method that represents each forget request as an unlearning task vector. By maintaining historical task vectors and performing sign-aware conflict-averse aggregation, \projectname{} suppresses conflicting update components that may weaken previous forgetting effects. Extensive experiments under both single-shot and continual settings show that \projectname{} outperforms baselines in terms of forgetting effectiveness, model fidelity, and forgetting persistence.
\end{abstract}

\section{Introduction}

Recent advances in vision-language models (VLMs), such as CLIP~\cite{radford2021learning}, have achieved remarkable success in aligning visual and textual representations, enabling a wide range of downstream applications. Despite their success, the large-scale data used for training VLMs inevitably introduces concerns regarding data privacy, copyright protection, the presence of undesirable or sensitive content, and several trustworthy AI challenges \cite{li2023trustworthy, dong2026allies, dong2025confound, dongtug}. To address these issues, machine unlearning~\cite{chundawat2023can,lin2023erm,chen2023boundary,lin2024gdr,tong2025robust,xiao2025reminiscence} has emerged as a promising direction for removing the influence of specific data from trained models without retraining from scratch. However, existing studies primarily focus on single-shot unlearning, where the forget request is assumed to be given at once. In practical VLM deployment, removal requests may arrive sequentially over time, as users, data owners, or regulators continuously request specific data, concepts, or associations to be removed. This motivates the study of \textit{continual machine unlearning}, which aims to process sequential removal requests in an online manner.

Although significant progress has been made in single-shot unlearning for VLMs~\cite{poppi2024safe,li2024single,cai2025targeted,yang2025cliperase,kravets2025zero,zhang2025targeted}, continual machine unlearning for VLMs remains largely underexplored. Directly extending single-shot methods to the continual setting is non-trivial. In single-shot unlearning, the model only needs to remove a fixed target set while maintaining its utility on retained data. In contrast, continual unlearning requires the model to repeatedly process new forgetting requests while preserving the effects of all previous unlearning steps. Since different forget sets may induce different, or even conflicting, parameter updates, a later unlearning step can weaken or partially reverse earlier forgetting effects. This leads to the \textit{knowledge re-emergence} problem, where previously forgotten information becomes accessible again after subsequent updates.

This problem is particularly important for VLMs. Unlike conventional classification models, VLMs encode visual and textual information in a shared cross-modal representation space. Different images, concepts, categories, and text prompts may be semantically correlated through this space. As a result, removing the influence of one forget set may affect related visual-textual associations, while later unlearning requests may unintentionally restore previously removed knowledge. Therefore, continual unlearning for VLMs should satisfy three key requirements: 
(i) \textit{effectiveness}, requiring the model to remove the influence of target data at each step; 
(ii) \textit{fidelity}, requiring the model to preserve performance on retained knowledge and downstream tasks; 
and (iii) \textit{persistence}, requiring previously forgotten information to remain forgotten after subsequent updates.

To address these challenges, we propose a continual machine unlearning method via \textbf{C}onflict-\textbf{A}verse \textbf{T}ask \textbf{A}rithmetic, namely \projectname. The key idea is to represent each forget request as an unlearning task vector and aggregate historical task vectors in a conflict-aware manner. Specifically, for each incoming forget set, \projectname{} first estimates the parameter direction associated with the target knowledge and takes its negative direction as the corresponding unlearning task vector. To reduce noisy and redundant updates, the task vector is sparsified by retaining the most influential parameter components. Then, \projectname{} performs sign-aware aggregation over historical task vectors, preserving direction-consistent components while suppressing conflicting ones. In this way, \projectname{} mitigates conflicts among sequential unlearning requests, reduces the risk of knowledge re-emergence, and naturally scales to longer sequences of unlearning requests by updating the aggregated task vector rather than repeatedly retraining from scratch. We summarize our contributions as follows:
\begin{itemize}
    \item We introduce the problem of continual machine unlearning for vision-language models, where sequential removal requests must be handled without compromising retained knowledge. We highlight knowledge re-emergence as a core challenge in this setting.

    \item We propose Conflict-averse Task Arithmetic, a task-vector-based unlearning framework that resolves conflicts among sequential forgetting directions through sign-aware aggregation, enabling stable continual unlearning.

    \item We conduct extensive evaluations across single-shot, continual, and long-sequence settings, showing that \projectname{} improves forgetting effectiveness, model fidelity, forgetting persistence, and scalability over strong baselines.
\end{itemize}

\section{Related Work}
\label{sec:related_work}

\partitle{Continual learning for VLMs}
Continual learning aims to enable models to learn from a sequence of tasks while mitigating catastrophic forgetting of previously acquired knowledge. Recent studies have extended continual learning to vision-language models (VLMs), focusing on sequential adaptation of multimodal representations and preservation of zero-shot generalization~\cite{jin2020visually,zheng2023preventing,yu2024boosting,ni2023continual,dong2025stabilizing,liu2025cclip}. Early work explored visually grounded continual learning by incrementally learning compositional phrase representations from streaming visual scenes~\cite{jin2020visually}. Subsequent studies investigated continual fine-tuning of CLIP and proposed strategies such as regularization to reduce zero-shot transfer degradation~\cite{zheng2023preventing}, representation constraints to preserve multimodal representations~\cite{ni2023continual,liu2025cclip}, and parameter-efficient adapters to mitigate task interference~\cite{yu2024boosting}.

Although these works are closely related to sequential adaptation in VLMs, their objective is fundamentally different from ours. Continual learning aims to acquire new knowledge without forgetting previous tasks, whereas continual unlearning aims to remove specified knowledge while preserving retained knowledge and preventing previously forgotten information from re-emerging. Therefore, methods designed for continual learning cannot be directly applied to continual unlearning.

\partitle{Machine unlearning for VLMs}
Machine unlearning aims to remove the influence of specified data or concepts from trained models without retraining from scratch. This is particularly important for VLMs such as CLIP~\cite{radford2021learning}, whose large-scale training data may contain private, copyrighted, or undesirable content. Recent works have explored unlearning in VLMs by selectively removing target concepts or associations while preserving general multimodal representations~\cite{poppi2024safe,cheng2024multidelete,cai2025targeted,yang2025cliperase,kravets2025zero,zhang2025targeted}. For example, MultiDelete~\cite{cheng2024multidelete}. separated embeddings associated with the forget set while maintaining unimodal representations for retained data. Other methods constrained unlearning gradients to selected layers~\cite{cai2025targeted}, disentangled and removed cross-modal associations in CLIP~\cite{yang2025cliperase}, or leveraged regularization and synthetic samples for zero-shot class unlearning~\cite{kravets2025zero}.

Despite their effectiveness, existing VLM unlearning methods primarily focus on single-shot settings, where the forget request is processed once. In practical deployment, however, removal requests may arrive sequentially over time. Naively applying single-shot methods to each request can cause conflicts among unlearning updates, leading to degraded model utility and knowledge re-emergence. In contrast, we propose Conflict-averse Task Arithmetic to aggregate sequential forgetting directions while suppressing conflicting updates for continual unlearning tasks.
\section{Preliminaries and Problem Analysis}
\label{sec:preliminaries}

\subsection{Problem Formulation}

Let $\mathcal{D} = \{(x_i, s_i, y_i)\}_{i=1}^{N}$ denote a multimodal dataset, where $x_i \in \mathcal{X}$ represents an image, $s_i \in \mathcal{S}$ denotes the associated text input (e.g., prompt or caption), and $y_i \in \mathcal{Y}$ is the supervision signal (e.g., label or target response). A vision-language model (VLM) $f_{\theta}$ parameterized by $\theta \in \mathbb{R}^d$ is trained via:
\begin{equation}
\theta^{*} = \arg\min_{\theta} \ \mathbb{E}_{(x,s,y)\sim \mathcal{D}} \ \ell(f_{\theta}(x,s), y),
\end{equation}
where $\ell(\cdot)$ denotes the task-specific loss function. We denote the trained model as the initial model before unlearning, i.e., $\theta^{(0)} = \theta^{*}$.

In this paper, we consider a \emph{continual machine unlearning} setting, where unlearning requests arrive sequentially in an online manner. Specifically, we define a sequence of multimodal forget sets:
\begin{equation}
\{\mathcal{D}_{u}^{(1)}, \mathcal{D}_{u}^{(2)}, \dots, \mathcal{D}_{u}^{(T)}\},
\end{equation}
where each $\mathcal{D}_{u}^{(t)} \subset \mathcal{D}$ corresponds to the forget set specified at time step $t$. At each step, the unlearning operator updates the current model according to the incoming forget set:
\begin{equation}
\theta_u^{(t)} = \mathcal{U}(\theta_u^{(t-1)}, \mathcal{D}_{u}^{(t)}).
\end{equation}
This process requires the model to incrementally remove newly specified knowledge while preserving retained knowledge and previously achieved forgetting effects. For each step $t$, we define the cumulative forget set and retained set as:
\begin{equation}
\mathcal{D}_{u}^{(\leq t)} = \bigcup_{j=1}^{t} \mathcal{D}_{u}^{(j)}, 
\quad
\mathcal{D}_{r}^{(t)} = \mathcal{D} \setminus \mathcal{D}_{u}^{(\leq t)}.
\end{equation}

Ideally, the unlearned model at each step should approximate the model retrained from scratch on the retained data:
\begin{equation}
\theta_u^{(t)} \approx 
\arg\min_{\theta} \ 
\mathbb{E}_{(x,s,y)\sim \mathcal{D}_{r}^{(t)}} 
\ell(f_{\theta}(x,s), y).
\end{equation}


\begin{wrapfigure}{r}{0.48\textwidth}
    \centering
    \vspace{-1.5em}
    \includegraphics[width=0.46\textwidth]{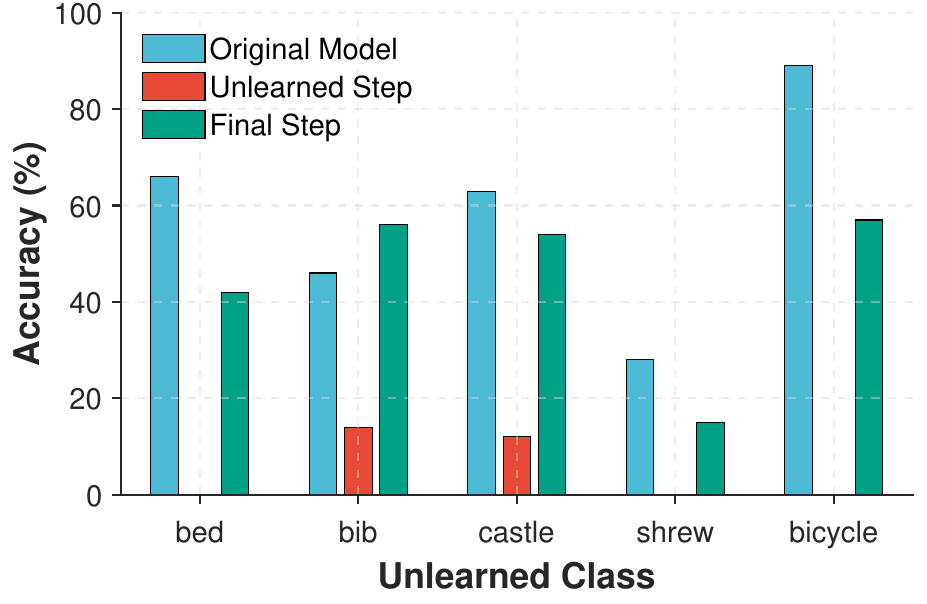}
    \caption{Illustration of knowledge re-emergence. Accuracy on target classes decreases after unlearning but partially recovers at the final step, indicating that previously forgotten knowledge may be restored by later updates.}
    \label{fig:reemergence}
    \vspace{-1em}
\end{wrapfigure}

\subsection{Knowledge Re-emergence Problem}

In continual machine unlearning, a critical challenge is the knowledge re-emergence problem, where knowledge removed in earlier unlearning steps becomes accessible again after subsequent requests, as illustrated in Fig.~\ref{fig:reemergence}. Specifically, the original model exhibits high accuracy on several target classes before unlearning. After each class is unlearned, the accuracy on the corresponding class drops significantly, indicating that the target knowledge has been effectively removed at that step. However, after subsequent unlearning requests are processed, the accuracy of previously unlearned classes partially recovers at the final step. This recovery suggests that later updates may unintentionally restore previously removed knowledge, demonstrating the knowledge re-emergence problem in continual machine unlearning.



\section{Proposed Method}
\label{sec:method}

\subsection{Overview}

As discussed in Section~\ref{sec:preliminaries}, a key challenge in continual machine unlearning is the re-emergence of previously forgotten knowledge after subsequent unlearning requests. We attribute this phenomenon to conflicts among unlearning updates induced by different forget sets. Inspired by task arithmetic for model editing~\cite{ilharco2023editing}, we represent each unlearning request as a task vector in parameter space, which provides an efficient way to compose multiple unlearning operations. However, directly aggregating these task vectors can introduce conflicting update directions, causing different unlearning effects to cancel or interfere with each other and potentially leading to knowledge re-emergence. Based on this observation, we propose \projectname, a conflict-averse task arithmetic framework for continual unlearning in VLMs. As illustrated in Fig.~\ref{fig:overview}, each incoming forget set is first converted into an unlearning task vector by reversing the parameter direction associated with the target knowledge. To reduce noisy and redundant updates, we sparsify each task vector through top-$k\%$ masking, retaining only the most influential parameter components. Given the historically sparse task vectors, \projectname{} estimates the dominant update direction at each parameter dimension and aggregates only direction-consistent components, while discarding conflicting ones that may weaken previous forgetting effects. The aggregated task vector is then applied to the pretrained model to obtain the unlearned model for the current step.

\begin{figure}[htbp]
    \centering
    \includegraphics[width=1\linewidth]{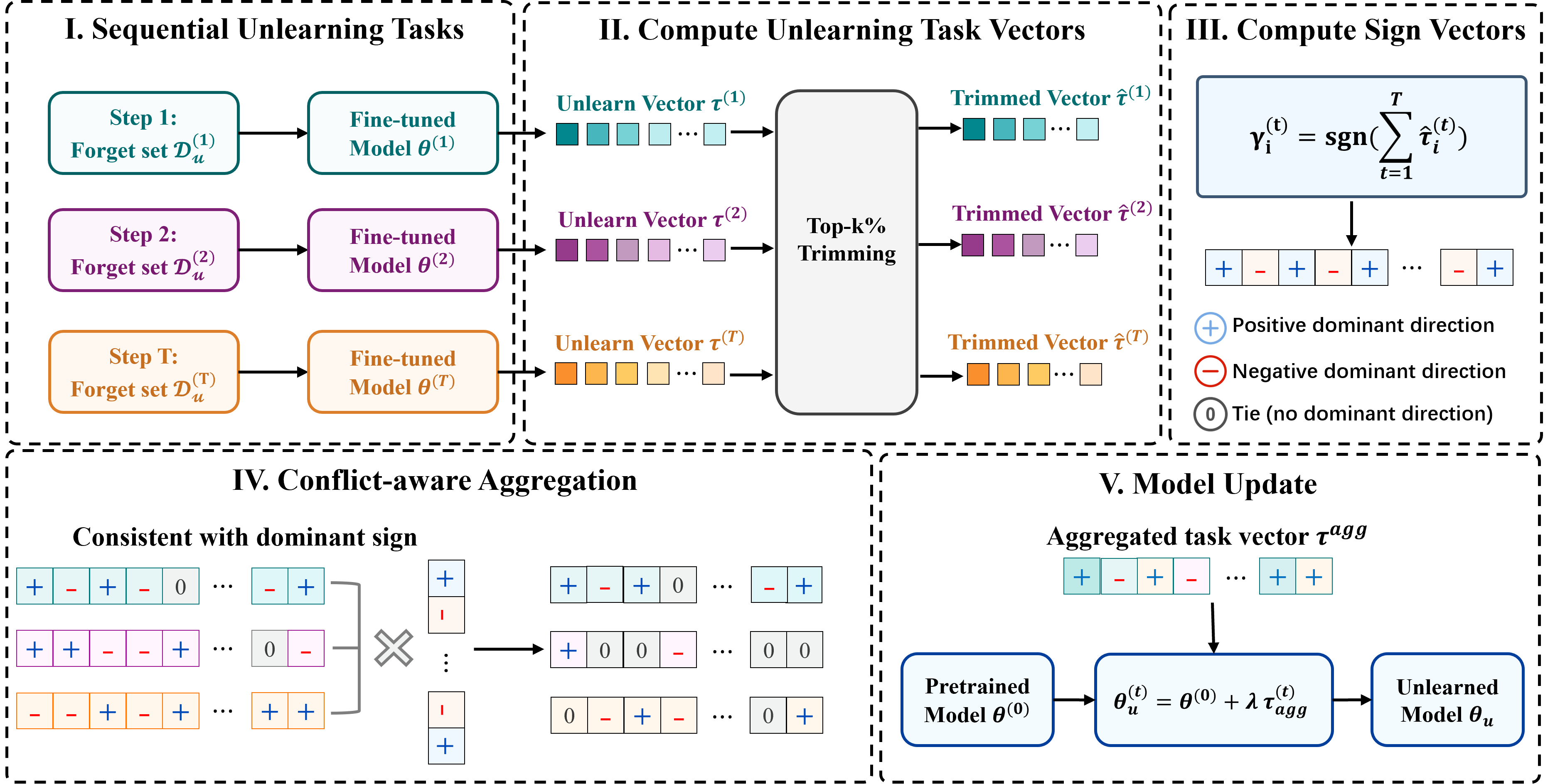}
    \caption{An overview of our proposed \projectname\ method. Each incoming forget set is converted into a sparse unlearning task vector through direction reversal and top-$k\%$ masking. Historical sparse task vectors are then aggregated by sign-aware conflict filtering, which suppresses conflicting components to mitigate the knowledge re-emergence issue.}
    \label{fig:overview}
\end{figure}

\partitle{Unlearning task vector}
For the $t$-th unlearning request, we estimate the parameter direction associated with the target knowledge in $\mathcal{D}_u^{(t)}$. Specifically, starting from the pretrained VLM $\theta^{(0)}$, we fine-tune the model on the current forget set:
\begin{equation}
    \theta_f^{(t)}
    =
    \arg\min_{\theta}
    \mathbb{E}_{(x,s,y)\sim \mathcal{D}_u^{(t)}}
    \ell(f_{\theta}(x,s), y),
\end{equation}
where the optimization is initialized from $\theta^{(0)}$, and $\ell(\cdot)$ denotes the task-specific VLM objective. The corresponding adaptation direction is defined as:
\begin{equation}
    \Delta \theta^{(t)}
    =
    \theta_f^{(t)} - \theta^{(0)}.
\end{equation}
Since this direction captures the parameter displacement associated with the current forget set, we define the unlearning task vector as its negative:
\begin{equation}
    \tau^{(t)}
    =
    -\Delta \theta^{(t)}
    =
    -(\theta_f^{(t)} - \theta^{(0)}).
\end{equation}

Although unlearning requests arrive sequentially, each task vector is constructed with respect to the same pretrained model $\theta^{(0)}$. This makes task vectors from different requests comparable in a shared parameter space, which is essential for subsequent conflict detection and aggregation.

\partitle{Sparse task vector masking}
Dense task vectors may contain many low-magnitude components that have limited impact on forgetting but can introduce noise during continual aggregation. To reduce such redundancy, we sparsify each task vector by retaining only the components with the largest magnitudes. For the unlearning task vector $\tau^{(t)}$, we define a binary mask $\mathbf{m}^{(t)} \in \{0,1\}^{d}$. Let $q_k(|\tau^{(t)}|)$ denote the magnitude threshold corresponding to the top-$k\%$ largest absolute values in $\tau^{(t)}$. The mask is defined as:
\begin{equation}
    \mathbf{m}^{(t)}_i
    =
    \mathbb{I}
    \left(
    |\tau_i^{(t)}|
    \geq
    q_k(|\tau^{(t)}|)
    \right),
\end{equation}
where $\mathbb{I}(\cdot)$ is the indicator function. The sparse task vector is then obtained by:
\begin{equation}
    \hat{\tau}^{(t)}
    =
    \mathbf{m}^{(t)} \odot \tau^{(t)},
\end{equation}
where $\odot$ denotes element-wise multiplication. This masking operation preserves the dominant parameter changes associated with the current unlearning request while filtering out low-impact components.

\partitle{Conflict-averse aggregation}
In continual unlearning, task vectors induced by different forget sets may conflict with each other. For two unlearning requests $a$ and $b$, a sign conflict occurs at parameter dimension $i$ when:
\begin{equation}
    \hat{\tau}_i^{(a)} \cdot \hat{\tau}_i^{(b)} < 0.
\end{equation}
Such conflicting components may cancel each other during aggregation, thereby weakening previous forgetting effects and causing knowledge re-emergence. To address this issue, we maintain a historical task-vector memory:
\begin{equation}
    \mathcal{M}_t
    =
    \{\hat{\tau}^{(1)}, \hat{\tau}^{(2)}, \dots, \hat{\tau}^{(t)}\},
\end{equation}
which stores the sparse task vectors observed up to step $t$. For each parameter dimension $i$, we estimate the dominant unlearning direction through a magnitude-weighted sign vote:
\begin{equation}
    \gamma_i^{(t)}
    =
    \operatorname{sgn}
    \left(
    \sum_{j=1}^{t}
    \hat{\tau}_i^{(j)}
    \right),
\end{equation}
where $\operatorname{sgn}(\cdot)$ is the sign function. If the summed update is zero, we set $\gamma_i^{(t)}=0$, indicating that no dominant direction exists at this dimension. We then select the components that are consistent with the dominant direction:
\begin{equation}
    A_i^{(t)}
    =
    \left\{
    j \in \{1,\dots,t\}
    \mid
    \hat{\tau}_i^{(j)} \neq 0,
    \operatorname{sgn}(\hat{\tau}_i^{(j)}) = \gamma_i^{(t)}
    \right\}.
\end{equation}
Components with signs inconsistent with $\gamma_i^{(t)}$ are treated as conflict-inducing updates and excluded from aggregation. The aggregated task vector at step $t$ is computed as:
\begin{equation}
    \tau^{i,(t)}_{\mathrm{agg}}
    =
    \begin{cases}
    \frac{1}{|A_i^{(t)}|}
    \sum_{j \in A_i^{(t)}}
    \hat{\tau}_i^{(j)},
    &
    \gamma_i^{(t)} \neq 0 \ \text{and} \ |A_i^{(t)}| > 0,
    \\[6pt]
    0,
    &
    \text{otherwise}.
    \end{cases}
\end{equation}
By aggregating only sign-consistent components, \projectname\ preserves the dominant unlearning direction while suppressing updates that may interfere with previous forgetting effects.

\partitle{Model update}
After obtaining the aggregated task vector $\tau^{(t)}_{\mathrm{agg}}$, we derive the unlearned model at step $t$ as:
\begin{equation}
    \theta_u^{(t)}
    =
    \theta^{(0)}
    +
    \lambda
    \tau^{(t)}_{\mathrm{agg}},
\end{equation}
where $\lambda$ controls the strength of unlearning. Although unlearning requests arrive sequentially, all task vectors are defined relative to the pretrained model $\theta^{(0)}$. Therefore, applying $\tau^{(t)}_{\mathrm{agg}}$ to $\theta^{(0)}$ ensures that the current and historical unlearning directions are integrated in a shared parameter space. The resulting model $\theta_u^{(t)}$ is expected to forget the cumulative forget set $\mathcal{D}_u^{(\leq t)}$ while preserving the knowledge associated with the retained data. The overall procedure of \projectname{} is summarized in Algorithm~\ref{alg:cata}.

\begin{algorithm}[htbp]
\caption{Conflict-averse Task Arithmetic}
\label{alg:cata}

Initialize task-vector memory $\mathcal{M}_0 \leftarrow \emptyset$\;

\For{$t = 1$ \KwTo $T$}{
    Fine-tune $\theta^{(0)}$ on $\mathcal{D}_u^{(t)}$ to obtain $\theta_f^{(t)}$\;
    
    Compute $\tau^{(t)} = -(\theta_f^{(t)}-\theta^{(0)})$ and apply top-$k\%$ masking to obtain $\hat{\tau}^{(t)}=\mathbf{m}^{(t)}\odot\tau^{(t)}$\;
    
    Update memory $\mathcal{M}_t \leftarrow \mathcal{M}_{t-1}\cup\{\hat{\tau}^{(t)}\}$\;
    
    \For{each parameter dimension $i$}{
        Compute dominant sign $\gamma_i^{(t)}=\operatorname{sgn}\!\left(\sum_{j=1}^{t}\hat{\tau}_i^{(j)}\right)$\;
        
        Select consistent indices $A_i^{(t)}=\{j\leq t \mid \hat{\tau}_i^{(j)}\neq 0,\operatorname{sgn}(\hat{\tau}_i^{(j)})=\gamma_i^{(t)}\}$\;
        
        Aggregate non-conflicting components:
        \[
        \tau_{\mathrm{agg},i}^{(t)}
        =
        \begin{cases}
        \frac{1}{|A_i^{(t)}|}\sum_{j\in A_i^{(t)}}\hat{\tau}_i^{(j)}, & |A_i^{(t)}|>0,\\
        0, & |A_i^{(t)}|=0.
        \end{cases}
        \]
    }
    
    Update $\theta_u^{(t)}=\theta^{(0)}+\lambda\tau_{\mathrm{agg}}^{(t)}$\;
}

\end{algorithm}


\section{Experiments}
\label{sec:experiments}

\partitle{Datasets, models, and baselines} In the main experiments, we evaluate \projectname{} on ImageNet-1K~\cite{deng2009imagenet} using four representative CLIP backbones, including ViT-B-32, ViT-L-14, RN50, and RN101. Additional results on CIFAR-10 and CIFAR-100~\cite{krizhevsky2009learning} are provided in the supplementary material. Furthermore, we compare \projectname{} with representative machine unlearning baselines, including FT~\cite{warnecke2023machine}, GA~\cite{thudi2022unrolling}, Fisher~\cite{golatkar2020eternal}, LIP~\cite{foster2024information}, EMMN~\cite{chundawat2023zero}, CLIP-LIP~\cite{kravets2025zero}, and TIFS~\cite{zhang2025targeted}. 

\partitle{Implementation details} 
To obtain the task vector, we fine-tune the CLIP visual encoder from the original pre-trained weights for 4 epochs using the AdamW optimizer with a learning rate of $1 \times 10^{-5}$ and a weight decay of $0.1$. The scaling factor $\lambda$ is set to $0.7$ and the top-$k$ trimming ratio is set to $0.3$ across all backbones. Additional implementation details and baseline configurations are provided in the supplementary material.

\partitle{Metrics} 
Following previous work \cite{cai2025targeted}, we report zero-shot classification accuracy on the targeted set (\textit{Target}), the retaining set (\textit{Retain}), and the full dataset (\textit{All}). To further evaluate model fidelity and generalization, we test on several unseen datasets, including Food~\cite{bossard2014food}, STL~\cite{coates2011analysis}, and ObjectNet~\cite{barbu2019objectnet}. We compute a normalized score as $\min(Acc_{\text{unlearn}} / Acc_{\text{original}}, 1)$, where the score is capped at 1 to indicate full preservation of the original model performance. For the \textit{Target}, we use $100 - \textit{Score}_{\text{Target}}$ when computing the average score to reflect the degree of forgetting. We report the \textit{Avg. Score} by averaging the normalized scores across all evaluated datasets. For continual unlearning, we use $\Delta$ to measure the difference between the final accuracy and that at the corresponding forgetting step for each target class and report the \textit{Avg. $\Delta$} by averaging $\Delta$ over all forgetting classes. 

\subsection{Main Results}

\partitle{Comparison experiments in continual unlearning}
As shown in Table~\ref{tab:continual-imagenet}, we compare \projectname{} with baselines under the continual unlearning setting on ImageNet-1K. The results show that \projectname{} achieves the best balance among these objectives on both ViT-B/32 and ViT-L/14. It consistently suppresses the accuracy of forgotten classes to near zero and maintains the lowest \textit{Avg. $\Delta$}, showing that the removed knowledge remains stably forgotten after subsequent updates. Meanwhile, \projectname{} obtains the highest \textit{Avg. Score} on both backbones, indicating better preservation of retained knowledge and downstream generalization. In contrast, the baselines typically fail in at least one aspect. FT preserves high retain and downstream performance but shows weak forgetting effectiveness, as the target accuracy often remains high after the corresponding forgetting step. GA achieves stronger forgetting than FT, but causes noticeable degradation in retained and downstream performance, especially under longer unlearning sequences. LIP aggressively suppresses target classes, but severely damages model utility, leading to very low retain and transfer performance. These results demonstrate that \projectname{} achieves more effective, persistent, and high-fidelity continual unlearning by better balancing target removal and utility preservation. Additional results on CIFAR-100 are provided in the supplementary material.

\begin{table*}[htbp]
\centering
\caption{Performance comparison on ImageNet-1K under continual unlearning. Underlined values indicate the class forgotten at the corresponding step. }
\label{tab:continual-imagenet}
\scriptsize
\setlength{\tabcolsep}{3.2pt}
\resizebox{\textwidth}{!}{
\begin{tabular}{ccc|ccccc|cccccccc}
\toprule[1pt]
\multirow{2}{*}{Backbone} 
& \multirow{2}{*}{Method} 
& \multirow{2}{*}{Step} 
& \multicolumn{5}{c|}{Target$\downarrow$} 
& \multirow{2}{*}{Retain$\uparrow$} 
& \multirow{2}{*}{All$\uparrow$} 
& \multirow{2}{*}{Food$\uparrow$} 
& \multirow{2}{*}{STL$\uparrow$} 
& \multirow{2}{*}{ObjectNet$\uparrow$} 
& \multirow{2}{*}{CIFAR-10$\uparrow$} 
& \multirow{2}{*}{Avg. $\Delta\downarrow$} 
& \multirow{2}{*}{Avg. Score$\uparrow$} \\
& & 
& Class 1
& Class 2
& Class 3 
& Class 4
& Class 5 
& & & & & & & & \\
\midrule[1pt]

\multirow{21}{*}{ViT-B/32}
& Original & Step 0 
& $64.00$ & $24.00$ & $32.00$ & $22.00$ & $90.00$ 
& $59.36$ & $59.29$ & $82.05$ & $97.36$ & $30.27$ & $88.79$ & -- & -- \\
\cmidrule{2-16}

& \multirow{5}{*}{GA \cite{thudi2022unrolling}}
& Step 1 & $\underline{2.00}$ & $16.00$ & $30.00$ & $24.00$ & $86.00$ & $57.69$ & $57.56$ & $81.80$ & $96.96$ & $28.43$ & $81.80$ & -- & -- \\
& & Step 2 & $0.00$ & $\underline{0.00}$ & $34.00$ & $44.00$ & $84.00$ & $55.25$ & $55.13$ & $79.30$ & $95.80$ & $27.79$ & $80.08$ & -- & -- \\
& & Step 3 & $2.00$ & $0.00$ & $\underline{2.00}$ & $36.00$ & $74.00$ & $52.04$ & $51.89$ & $77.00$ & $95.47$ & $25.92$ & $78.58$ & -- & -- \\
& & Step 4 & $0.00$ & $0.00$ & $0.00$ & $\underline{12.00}$ & $62.00$ & $48.34$ & $48.17$ & $74.47$ & $95.16$ & $23.56$ & $78.20$ & -- & -- \\
& & Step 5 & $0.00$ & $0.00$ & $0.00$ & $14.00$ & $\underline{10.00}$ & $44.01$ & $43.81$ & $70.10$ & $94.74$ & $20.02$ & $95.70$ & $1.20$ & $84.54$ \\
\cmidrule{2-16}

& \multirow{5}{*}{FT \cite{warnecke2023machine}}
& Step 1 & $\underline{56.00}$ & $14.00$ & $50.00$ & $26.00$ & $88.00$ & $62.45$ & $62.37$ & $82.17$ & $97.74$ & $30.46$ & $89.61$ & -- & -- \\
& & Step 2 & $52.00$ & $\underline{12.00}$ & $46.00$ & $28.00$ & $90.00$ & $63.26$ & $63.17$ & $81.22$ & $97.92$ & $30.55$ & $88.73$ & -- & -- \\
& & Step 3 & $42.00$ & $10.00$ & $\underline{50.00}$ & $30.00$ & $88.00$ & $63.50$ & $63.40$ & $79.93$ & $97.60$ & $30.58$ & $87.24$ & -- & -- \\
& & Step 4 & $46.00$ & $14.00$ & $46.00$ & $\underline{30.00}$ & $82.00$ & $63.63$ & $63.53$ & $78.72$ & $97.55$ & $30.71$ & $86.92$ & -- & -- \\
& & Step 5 & $40.00$ & $10.00$ & $46.00$ & $34.00$ & $\underline{88.00}$ & $63.85$ & $63.75$ & $78.02$ & $97.47$ & $30.78$ & $85.91$ & $5.20$ & $55.32$ \\
\cmidrule{2-16}

& \multirow{5}{*}{LIP \cite{foster2024information}}
& Step 1 & $\underline{0.00}$ & $0.00$ & $0.00$ & $0.00$ & $0.00$ & $0.71$ & $0.71$ & $12.84$ & $41.09$ & $2.27$ & $35.16$ & -- & -- \\
& & Step 2 & $0.00$ & $\underline{0.00}$ & $0.00$ & $0.00$ & $0.00$ & $0.15$ & $0.15$ & $1.06$ & $9.97$ & $0.31$ & $11.97$ & -- & -- \\
& & Step 3 & $0.00$ & $0.00$ & $\underline{0.00}$ & $0.00$ & $0.00$ & $0.17$ & $0.17$ & $1.10$ & $7.19$ & $0.29$ & $12.10$ & -- & -- \\
& & Step 4 & $32.00$ & $0.00$ & $0.00$ & $\underline{0.00}$ & $0.00$ & $0.10$ & $0.14$ & $0.91$ & $10.28$ & $0.26$ & $10.44$ & -- & -- \\
& & Step 5 & $74.00$ & $0.00$ & $0.00$ & $0.00$ & $\underline{0.00}$ & $0.01$ & $0.09$ & $0.94$ & $10.40$ & $0.25$ & $10.55$ & $14.80$ & $37.19$ \\
\cmidrule{2-16}

& \multirow{5}{*}{\projectname\ (ours)}
& Step 1 & $\underline{0.00}$ & $20.00$ & $26.00$ & $28.00$ & $78.00$ & $56.22$ & $56.09$ & $78.18$ & $96.83$ & $27.23$ & $82.44$ & -- & -- \\
& & Step 2 & $0.00$ & $\underline{0.00}$ & $24.00$ & $36.00$ & $70.00$ & $55.17$ & $55.02$ & $78.42$ & $95.89$ & $27.17$ & $81.19$ & -- & -- \\
& & Step 3 & $0.00$ & $0.00$ & $\underline{0.00}$ & $22.00$ & $58.00$ & $54.73$ & $54.54$ & $78.06$ & $96.20$ & $27.06$ & $81.64$ & -- & -- \\
& & Step 4 & $0.00$ & $0.00$ & $0.00$ & $\underline{0.00}$ & $66.00$ & $54.07$ & $53.86$ & $78.04$ & $96.38$ & $26.85$ & $85.13$ & -- & -- \\
& & Step 5 & $0.00$ & $0.00$ & $0.00$ & $0.00$ & $\underline{2.00}$ & $54.33$ & $54.06$ & $77.20$ & $96.69$ & $26.42$ & $84.01$ & $\textbf{0.00}$ & $\textbf{95.98}$ \\

\midrule[1pt]

\multirow{21}{*}{ViT-L/14}
& Original & Step 0
& $38.00$ & $46.00$ & $34.00$ & $52.00$ & $88.00$ 
& $71.72$ & $71.62$ & $92.05$ & $99.41$ & $51.86$ & $95.32$ & -- & -- \\
\cmidrule{2-16}

& \multirow{5}{*}{GA \cite{thudi2022unrolling}}
& Step 1 & $\underline{6.00}$ & $82.00$ & $32.00$ & $56.00$ & $76.00$ & $69.95$ & $69.85$ & $91.50$ & $99.35$ & $51.90$ & $90.48$ & -- & -- \\
& & Step 2 & $2.00$ & $\underline{12.00}$ & $36.00$ & $48.00$ & $26.00$ & $64.96$ & $64.76$ & $88.91$ & $97.88$ & $50.45$ & $87.30$ & -- & -- \\
& & Step 3 & $0.00$ & $4.00$ & $\underline{2.00}$ & $34.00$ & $0.00$ & $43.76$ & $43.58$ & $46.30$ & $73.88$ & $45.70$ & $73.81$ & -- & -- \\
& & Step 4 & $0.00$ & $2.00$ & $0.00$ & $\underline{6.00}$ & $0.00$ & $23.35$ & $23.24$ & $11.69$ & $44.90$ & $34.15$ & $48.08$ & -- & -- \\
& & Step 5 & $0.00$ & $2.00$ & $0.00$ & $10.00$ & $\underline{2.00}$ & $6.04$ & $6.02$ & $5.30$ & $42.27$ & $26.29$ & $43.86$ & $4.40$ & $57.81$ \\
\cmidrule{2-16}

& \multirow{5}{*}{FT \cite{warnecke2023machine}}
& Step 1 & $\underline{58.00}$ & $18.00$ & $32.00$ & $62.00$ & $78.00$ & $75.36$ & $75.23$ & $91.72$ & $99.28$ & $52.39$ & $93.87$ & -- & -- \\
& & Step 2 & $28.00$ & $\underline{46.00}$ & $50.00$ & $64.00$ & $66.00$ & $75.81$ & $75.68$ & $90.36$ & $99.30$ & $50.75$ & $87.99$ & -- & -- \\
& & Step 3 & $62.00$ & $16.00$ & $\underline{52.00}$ & $62.00$ & $68.00$ & $76.40$ & $76.28$ & $90.52$ & $99.41$ & $52.09$ & $92.16$ & -- & -- \\
& & Step 4 & $54.00$ & $14.00$ & $50.00$ & $\underline{68.00}$ & $96.00$ & $76.29$ & $76.19$ & $90.08$ & $99.29$ & $51.79$ & $91.54$ & -- & -- \\
& & Step 5 & $66.00$ & $56.00$ & $48.00$ & $56.00$ & $\underline{88.00}$ & $76.69$ & $76.62$ & $89.23$ & $99.25$ & $51.87$ & $91.94$ & $6.80$ & $42.08$ \\
\cmidrule{2-16}

& \multirow{5}{*}{LIP \cite{foster2024information}}
& Step 1 & $\underline{0.00}$ & $0.00$ & $0.00$ & $0.00$ & $0.00$ & $2.30$ & $2.29$ & $18.55$ & $55.75$ & $14.27$ & $31.16$ & -- & -- \\
& & Step 2 & $0.00$ & $\underline{0.00}$ & $0.00$ & $0.00$ & $0.00$ & $0.17$ & $0.17$ & $1.81$ & $10.86$ & $2.90$ & $10.06$ & -- & -- \\
& & Step 3 & $18.00$ & $0.00$ & $\underline{0.00}$ & $0.00$ & $0.00$ & $0.07$ & $0.08$ & $1.14$ & $9.07$ & $2.54$ & $9.75$ & -- & -- \\
& & Step 4 & $36.00$ & $0.00$ & $0.00$ & $\underline{0.00}$ & $0.00$ & $0.07$ & $0.11$ & $0.82$ & $9.22$ & $2.52$ & $9.71$ & -- & -- \\
& & Step 5 & $48.00$ & $0.00$ & $0.00$ & $0.00$ & $\underline{0.00}$ & $0.04$ & $0.09$ & $0.71$ & $9.22$ & $2.50$ & $9.75$ & $9.60$ & $36.27$ \\
\cmidrule{2-16}

& \multirow{5}{*}{\projectname\ (ours)}
& Step 1 & $\underline{0.00}$ & $66.00$ & $36.00$ & $46.00$ & $88.00$ & $70.66$ & $70.55$ & $91.22$ & $99.41$ & $51.19$ & $89.91$ & -- & -- \\
& & Step 2 & $0.00$ & $\underline{0.00}$ & $40.00$ & $50.00$ & $90.00$ & $70.69$ & $70.52$ & $91.23$ & $99.36$ & $50.42$ & $90.89$ & -- & -- \\
& & Step 3 & $0.00$ & $0.00$ & $\underline{0.00}$ & $52.00$ & $86.00$ & $70.71$ & $70.49$ & $91.13$ & $99.40$ & $50.59$ & $92.37$ & -- & -- \\
& & Step 4 & $0.00$ & $0.00$ & $0.00$ & $\underline{0.00}$ & $90.00$ & $69.92$ & $69.66$ & $90.99$ & $99.22$ & $50.17$ & $93.65$ & -- & -- \\
& & Step 5 & $0.00$ & $0.00$ & $0.00$ & $8.00$ & $\underline{10.00}$ & $70.22$ & $69.89$ & $91.11$ & $99.31$ & $50.15$ & $93.26$ & ${\textbf{1.60}}$ & ${\textbf{96.56}}$ \\

\bottomrule[1pt]
\end{tabular}}
\vspace{-1em}
\end{table*}

\begin{table*}[htbp]
\centering
\caption{Performance comparison on ImageNet-1K in single-shot unlearning.}
\resizebox{\textwidth}{!}{
\begin{tabular}{ccccccccccc}
\toprule
\multirow{2}{*}{Backbone} & \multirow{2}{*}{Method} & \multicolumn{3}{c}{ImageNet} & \multirow{2}{*}{Food$\uparrow$} & \multirow{2}{*}{STL$\uparrow$} & \multirow{2}{*}{ObjectNet$\uparrow$} & \multirow{2}{*}{CIFAR-10$\uparrow$} & \multirow{2}{*}{Avg. Score$\uparrow$} \\ \cmidrule{3-5}
& & Target$\downarrow$ & Retain$\uparrow$ & All$\uparrow$ & & & &  &  \\ 

\midrule

\multirow{9}{*}{RN50}
& Original  
            & $54.10$ & $66.61$ & $65.37$ & $76.49$ 
            & $93.75$ & $25.83$ & $68.84$ & --\\ \cmidrule{2-10}
& FT  \cite{warnecke2023machine}        
            & $4.00_{5.88}$  & $0.26_{0.53}$  & $24.80_{47.75}$ & $30.34_{39.67}$ & $67.05_{71.52}$ & $9.58_{37.09}$  & $12.26_{17.81}$ & $44.07$ \\
& GA \cite{thudi2022unrolling}       
            & $14.50_{21.32}$ & $35.40_{72.73}$ & $35.39_{68.14}$ & $51.39_{67.18}$ & $77.76_{82.94}$ & $14.18_{54.90}$ & $14.08_{20.45}$ & $63.58$ \\
& Fisher \cite{golatkar2020eternal}    
            & $0.20_{0.29}$  & $0.34_{0.70}$  & $1.09_{2.10}$  & $0.23_{0.30}$  & $10.69_{11.40}$ & $1.14_{4.41}$  & $10.12_{14.70}$ & $19.04$\\
& LIP \cite{foster2024information}      
            & $0.95_{1.40}$  & $1.75_{3.60}$  & $1.14_{2.19}$  & $0.23_{0.30}$  & $10.50_{11.20}$ & $0.59_{2.28}$  & $10.46_{15.19}$ & $19.06$\\
& EMMN \cite{chundawat2023zero}     
            & $10.00_{14.71}$ & $36.67_{75.34}$ & $32.25_{62.09}$ & $47.48_{62.07}$ & $65.21_{69.56}$ & $9.65_{37.36}$  & $11.75_{17.07}$ & $58.40$ \\
& CLIP-LIP \cite{kravets2025zero} 
            & $37.00_{54.41}$ & $52.77_{100.00}$ & $53.78_{100.00}$ & $76.47_{99.97}$ & $93.80_{100.00}$ & $25.80_{99.88}$ & $67.77_{98.45}$ & $91.98$ \\
& TIFS \cite{zhang2025targeted}    
            & $0.00_{0.00}$  & $44.67_{91.78}$ & $47.34_{91.14}$ & $79.00_{100.00}$ & $83.05_{88.59}$ & $16.95_{65.62}$ & $83.80_{100.00}$ & $91.01$ \\ \cmidrule{2-10}
& \projectname \ (ours)       
            & $2.00_{2.94}$  & $51.73_{100.00}$ & $51.68_{99.50}$ & $76.26_{99.70}$ & $94.03_{100.00}$ & $25.22_{97.64}$ & $65.94_{95.79}$ & $\textbf{98.53}$ \\

\midrule
\multirow{8}{*}{RN101}
& Original  
            & $80.00$ & $54.67$ & $54.33$ & $81.16$ & $96.46$ & $29.16$ & $73.82$ & --\\ \cmidrule{2-10}
& FT  \cite{warnecke2023machine}        
            & ${0.00}_{0.00}$ & $0.51_{0.93}$  & $27.62_{50.84}$ & $38.54_{47.49}$ & $74.75_{77.49}$ & $13.26_{45.47}$ & $15.78_{21.38}$ & $49.09$\\
& GA  \cite{thudi2022unrolling}      
            & $32.00_{40.00}$ & $38.76_{70.90}$ & $38.75_{71.32}$ & $58.44_{72.01}$ & $84.45_{87.55}$ & $18.62_{63.85}$ & $18.73_{25.37}$ & $64.43$\\
& Fisher \cite{golatkar2020eternal}    
            & $1.12_{1.40}$  & $0.38_{0.70}$  & $0.27_{0.50}$  & $0.24_{0.30}$  & $9.94_{10.30}$  & $0.41_{1.41}$  & $13.44_{18.21}$ & $18.57$\\
& LIP  \cite{foster2024information}     
            & $0.24_{0.30}$  & $0.27_{0.49}$  & $1.74_{3.20}$  & $0.16_{0.20}$  & $11.19_{11.60}$ & $0.79_{2.71}$  & $11.89_{16.11}$ & $19.14$ \\
& EMMN \cite{chundawat2023zero}    
            & $68.00_{85.00}$ & $29.33_{53.65}$ & $26.43_{48.65}$ & $42.44_{52.29}$ & $72.46_{75.12}$ & $11.97_{41.05}$ & $17.31_{23.45}$ & $44.17$\\
& CLIP-LIP \cite{kravets2025zero} & 
            $2.00_{2.50}$  & $50.84_{92.99}$ & $51.49_{94.77}$ & $78.50_{96.72}$ & $96.42_{99.96}$ & $27.66_{94.86}$ & $70.08_{94.93}$ & $95.96$ \\ 
& TIFS \cite{zhang2025targeted} 
            & $0.00_{0.00}$ & $50.17_{91.77}$ & $49.51_{91.13}$ & $70.50_{86.87}$ & $85.45_{88.59}$ & $19.13_{65.60}$ & $89.86_{100.00}$ & $89.14$\\ \cmidrule{2-10}
& \projectname \ (ours)       
            & $0.00_{0.00}$  & $53.06_{97.06}$ & $52.96_{97.48}$ & $78.98_{97.31}$ & $96.45_{100.00}$ & $27.27_{93.52}$ & $72.04_{97.59}$ & $\textbf{97.56}$ \\

\bottomrule
\end{tabular}}

\label{tab:single-shot imagenet}
\end{table*}

\partitle{Comparison experiments in single-shot unlearning}
Table~\ref{tab:single-shot imagenet} compares \projectname{} with representative unlearning methods on ImageNet-1K under the single-shot setting. The results show that \projectname{} achieves the best \textit{Avg. Score} on both RN50 and RN101, indicating that it preserves the original model utility more effectively across in-domain and transfer benchmarks after unlearning. Compared with FT, GA, Fisher, and LIP, which often reduce target accuracy at the cost of severe degradation on retain classes and downstream datasets, \projectname{} maintains substantially stronger retention and generalization performance. CLIP-LIP preserves model utility well but leaves higher target accuracy in some settings, suggesting insufficient forgetting. TIFS achieves strong target suppression, but its average score is lower due to performance drops on retain or transfer benchmarks. In contrast, \projectname{} achieves low target accuracy while maintaining model performance, demonstrating a better balance between forgetting effectiveness and model fidelity in single-shot unlearning. Additional results on CIFAR-10 are provided in the supplementary material.

\begin{figure}[htbp]
    \centering
    \includegraphics[width=\columnwidth]{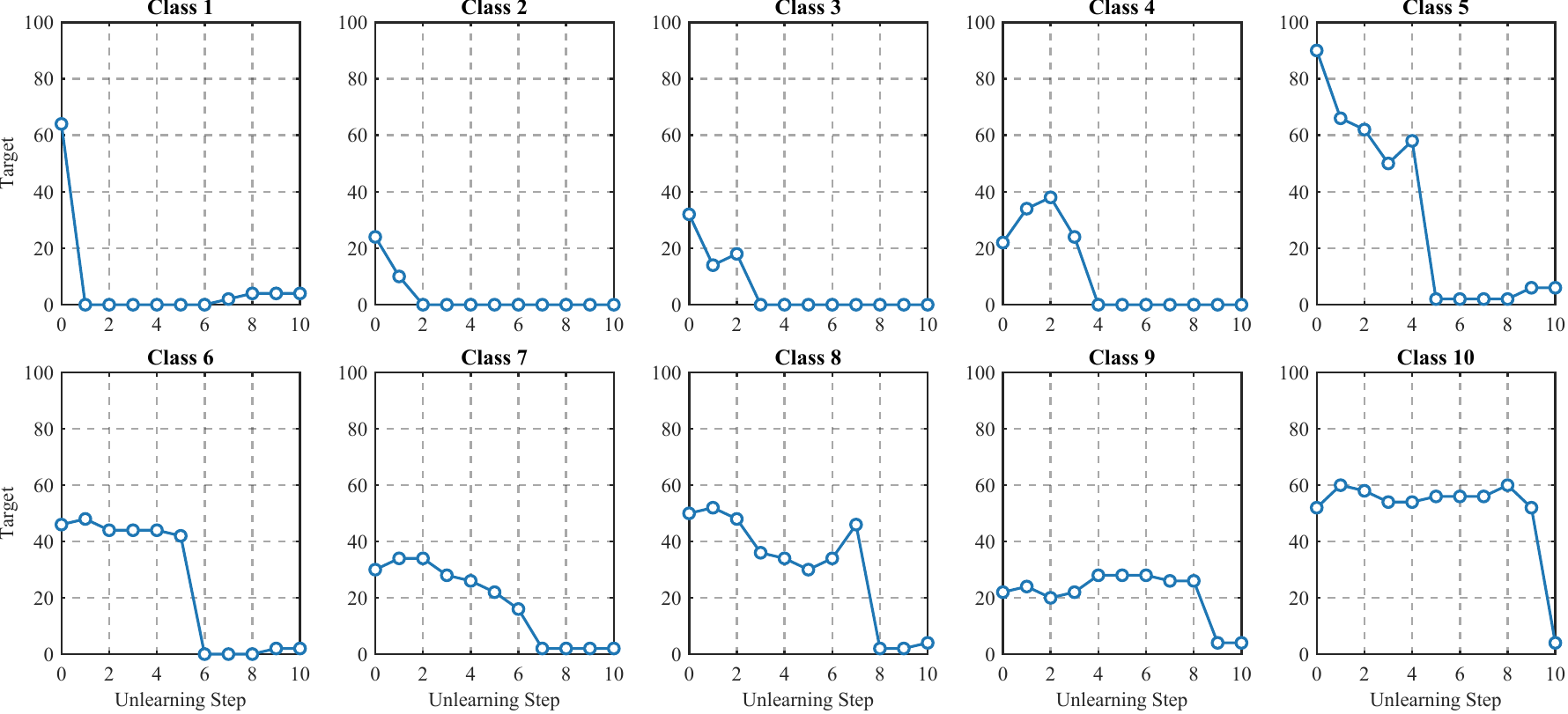}
    \caption{Scalability evaluation on ImageNet-1K in continual unlearning. Each subplot shows the target accuracy of one removed class across continual unlearning steps.}
    \label{fig:scalability}
\end{figure}

\partitle{Scalability evaluation}
We evaluate scalability by progressively increasing the number of unlearning steps, with each step removing one target class. Fig.~\ref{fig:scalability} reports the target accuracy of each removed class across the unlearning sequence. The accuracy of most target classes quickly drops to near zero after the corresponding unlearning step and remains low in subsequent steps, indicating stable forgetting without obvious knowledge re-emergence. These results suggest that \projectname\ can maintain effective target removal as the number of unlearning requests increases. Besides, the detailed step-wise results in the supplementary material (Table~\ref{tab:vitB_in10_lam09}) indicate that the retain and overall accuracies remain stable as the number of unlearning steps increases. This demonstrates that \projectname{} scales to longer unlearning sequences without significant degradation in model utility.

\begin{table*}[htbp]
\centering
\caption{Ablation study of aggregation strategies on ImageNet-1K in continual unlearning. Underlined values indicate the class forgotten at the corresponding step.}
\label{tab:aggregation}
\scriptsize
\setlength{\tabcolsep}{3.5pt}
\resizebox{\textwidth}{!}{
\begin{tabular}{cc|ccccc|cccccccc}
\toprule[1pt]
\multirow{2}{*}{Aggregation}
& \multirow{2}{*}{Step} 
& \multicolumn{5}{c|}{Target$\downarrow$} 
& \multirow{2}{*}{Retain$\uparrow$} 
& \multirow{2}{*}{All$\uparrow$} 
& \multirow{2}{*}{Food$\uparrow$} 
& \multirow{2}{*}{STL$\uparrow$} 
& \multirow{2}{*}{ObjectNet$\uparrow$} 
& \multirow{2}{*}{CIFAR-10$\uparrow$} 
& \multirow{2}{*}{Avg. $\Delta\downarrow$} 
& \multirow{2}{*}{Avg. Score$\uparrow$} \\
& 
& Class 1 
& Class 2  
& Class 3 
& Class 4
& Class 5 
& & & & & & & & \\
\midrule[1pt]

Original & Step 0 
& $64.00$ & $24.00$ & $32.00$ & $22.00$ & $90.00$ 
& $59.36$ & $59.29$ & $82.05$ & $97.36$ & $30.27$ & $88.79$ 
& -- & -- \\
\cmidrule{1-15}

\multirow{5}{*}{Naive Avg.}
& Step 1 & $\underline{0.00}$ & $18.00$ & $20.00$ & $32.00$ & $72.00$ 
& $55.84$ & $55.70$ & $77.52$ & $96.54$ & $27.14$ & $81.28$ 
& -- & -- \\
& Step 2 & $0.00$ & $\underline{0.00}$ & $28.00$ & $34.00$ & $72.00$ 
& $56.57$ & $56.42$ & $80.29$ & $96.20$ & $28.32$ & $82.22$ 
& -- & -- \\
& Step 3 & $0.00$ & $0.00$ & $\underline{0.00}$ & $26.00$ & $68.00$ 
& $56.87$ & $56.68$ & $80.53$ & $96.71$ & $28.52$ & $83.72$ 
& -- & -- \\
& Step 4 & $4.00$ & $0.00$ & $0.00$ & $\underline{2.00}$ & $76.00$ 
& $57.25$ & $57.05$ & $81.28$ & $97.08$ & $28.82$ & $86.66$ 
& -- & -- \\
& Step 5 & $16.00$ & $2.00$ & $0.00$ & $8.00$ & $\underline{34.00}$ 
& $57.75$ & $57.53$ & $81.34$ & $97.19$ & $28.82$ & $86.39$ 
& $4.80$ & $88.94$ \\
\cmidrule{1-15}

\multirow{5}{*}{Ours}
& Step 1 & $\underline{0.00}$ & $20.00$ & $26.00$ & $28.00$ & $78.00$ 
& $56.22$ & $56.09$ & $78.18$ & $96.83$ & $27.23$ & $82.44$ 
& -- & -- \\
& Step 2 & $0.00$ & $\underline{0.00}$ & $24.00$ & $36.00$ & $70.00$ 
& $55.17$ & $55.02$ & $78.42$ & $95.89$ & $27.17$ & $81.19$ 
& -- & -- \\
& Step 3 & $0.00$ & $0.00$ & $\underline{0.00}$ & $22.00$ & $58.00$ 
& $54.73$ & $54.54$ & $78.06$ & $96.20$ & $27.06$ & $81.64$ 
& -- & -- \\
& Step 4 & $0.00$ & $0.00$ & $0.00$ & $\underline{0.00}$ & $66.00$ 
& $54.07$ & $53.86$ & $78.04$ & $96.38$ & $26.85$ & $85.13$ 
& -- & -- \\
& Step 5 & $0.00$ & $0.00$ & $0.00$ & $0.00$ & $\underline{2.00}$ 
& $54.33$ & $54.06$ & $77.20$ & $96.69$ & $26.42$ & $84.01$ &
${\textbf{0.00}}$ & ${\textbf{95.98}}$ \\

\bottomrule[1pt]
\end{tabular}}
\end{table*}

\subsection{Ablation Study}

\partitle{Impact of the aggregation strategy} 
Table~\ref{tab:aggregation} evaluates the role of conflict-averse aggregation. Naive averaging directly combines task vectors and therefore suffers from conflicting update directions, leading to incomplete forgetting and noticeable knowledge re-emergence. In contrast, our conflict-averse aggregation suppresses inconsistent components, achieving more complete forgetting and reducing Avg. $\Delta$ to zero. Although it introduces a mild drop in retained performance, it yields a higher Avg. Score, demonstrating a better balance between effectiveness, fidelity, and persistence. These results show that direct aggregation is insufficient for continual unlearning, and conflict-aware aggregation is essential for stable sequential forgetting.

\partitle{Impact of the scaling factor $\lambda$} 
As shown in Fig. \ref{fig:lambda}, we evaluate the effect of the scaling factor $\lambda$ in continual unlearning. As $\lambda$ increases, the \textit{Target} accuracy and Avg.$\Delta$ are progressively reduced, indicating stronger forgetting and less knowledge re-emergence. Meanwhile, the overall accuracy stays relatively stable across a wide range of $\lambda$, whereas the Avg. Score drops when $\lambda$ becomes too large, suggesting a loss of model fidelity. These results show that $\lambda$ governs the balance between forgetting effectiveness and performance preservation. We therefore use $\lambda=0.7$ in our experiments.

\vspace{-1em}
\begin{figure}[htbp]
    \centering
    \subfloat[Forgetting effectiveness]{
    \includegraphics[width=0.45\columnwidth]{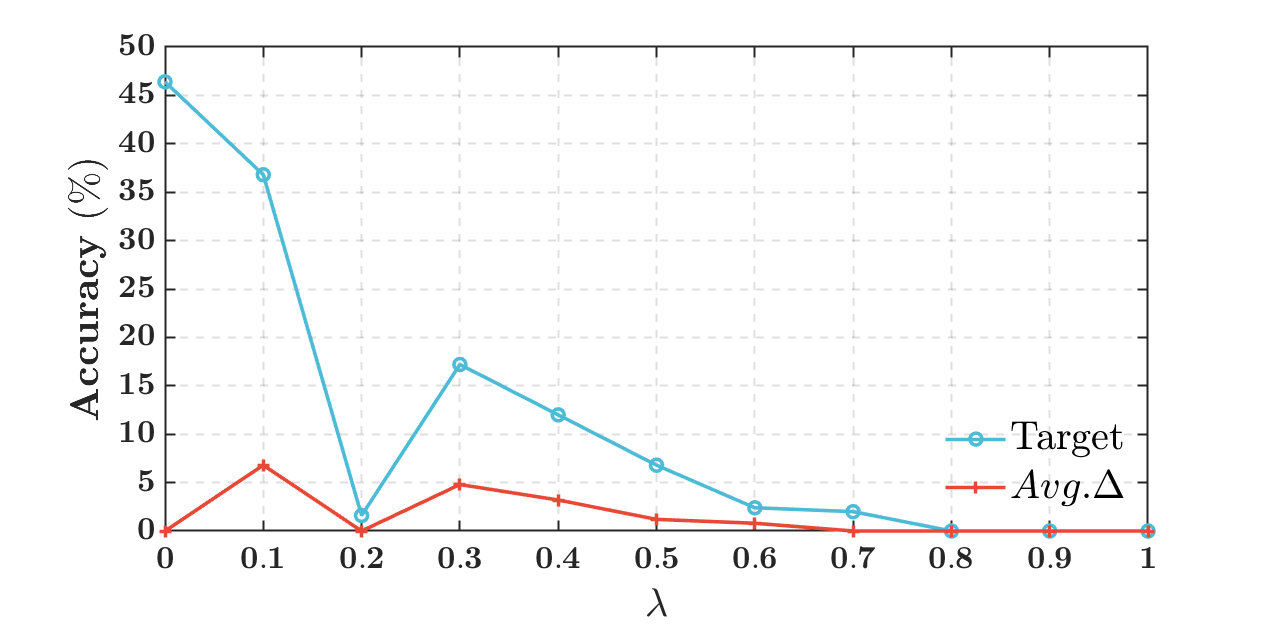}}
    \quad
        \subfloat[Model fidelity]{
    \includegraphics[width=0.45\columnwidth]{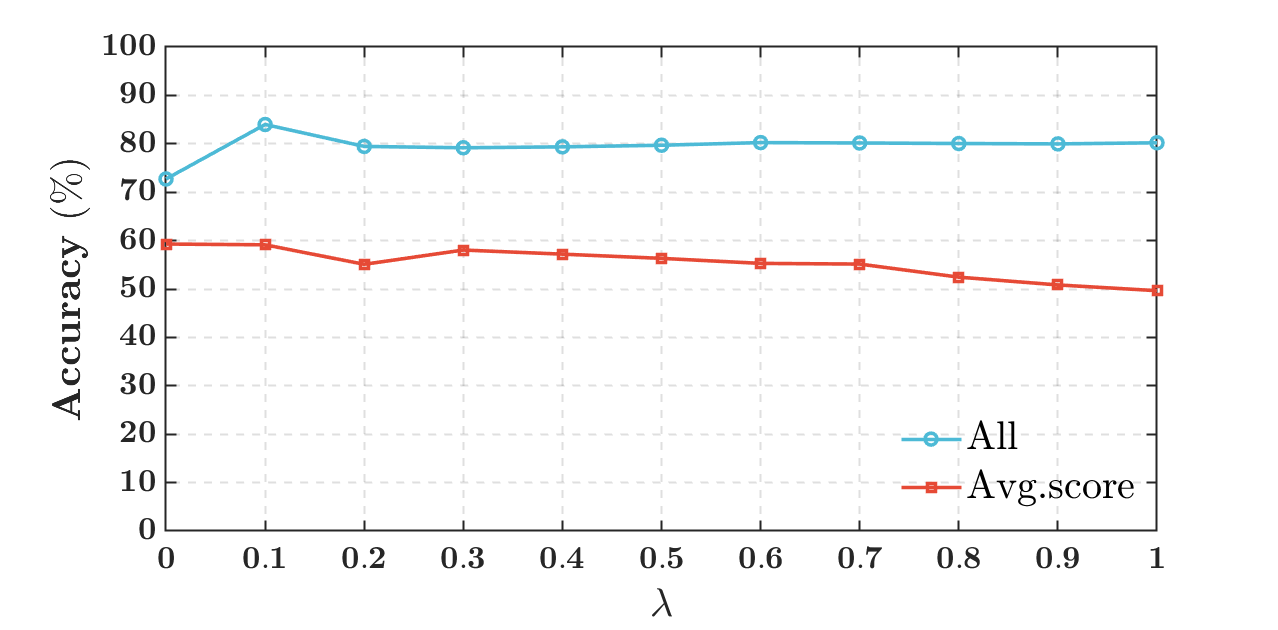}
    }
    \caption{Ablation study of the scaling factor $\lambda$ on ImageNet-1K in continual unlearning.}
    \label{fig:lambda}
\end{figure}

\partitle{Impact of the top-$k\%$ selection} 
Fig. \ref{fig:top-k}  shows the effect of the top-$k\%$ selection in task vector trimming. When $k$ is too small, the task vector becomes overly sparse, leading to insufficient unlearning and higher target accuracy. As $k$ increases, both \textit{Target} and Avg.$\Delta$ quickly decrease and remain close to zero, while \textit{All} and Avg. Score stay relatively stable. These results indicate that effective continual unlearning can already be achieved with a relatively small top-$k\%$, implying that the task vectors are highly sparse and incur very low storage overhead in practice. Based on this observation, we set $k=0.3$ in our evaluation experiments.

\vspace{-1em}
\begin{figure}[htbp]
    \centering
    \subfloat[Forgetting effectiveness]{
    \includegraphics[width=0.45\columnwidth]{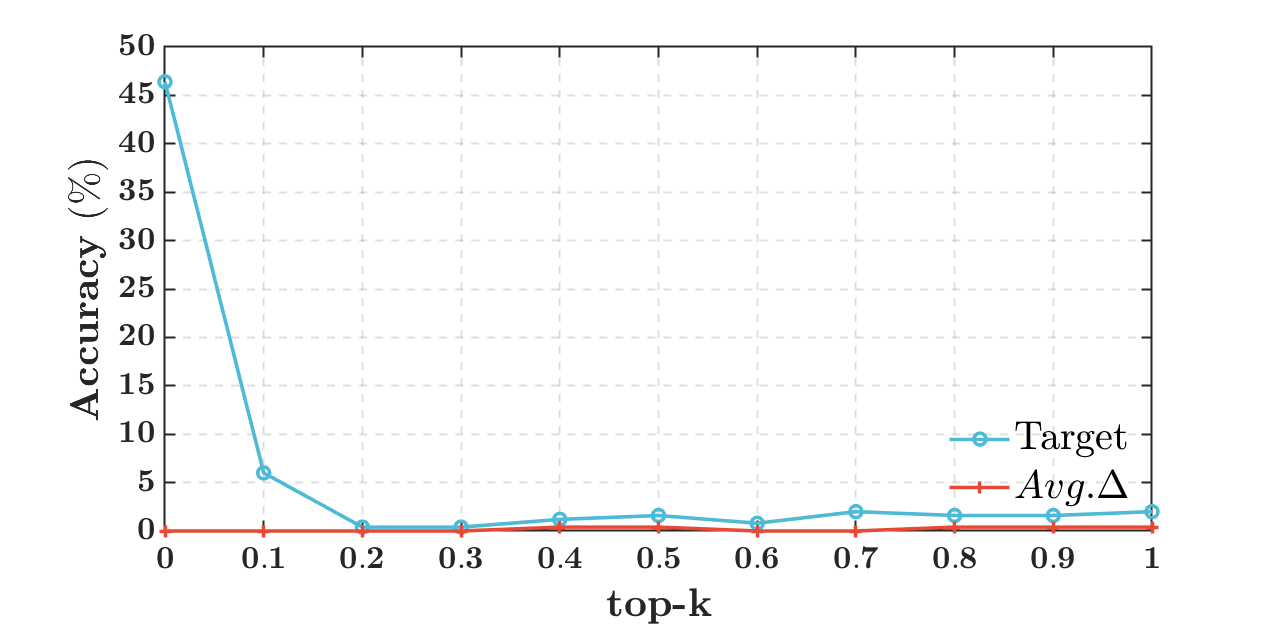}
    }
    \quad
    \subfloat[Model fidelity]{
    \includegraphics[width=0.45\columnwidth]{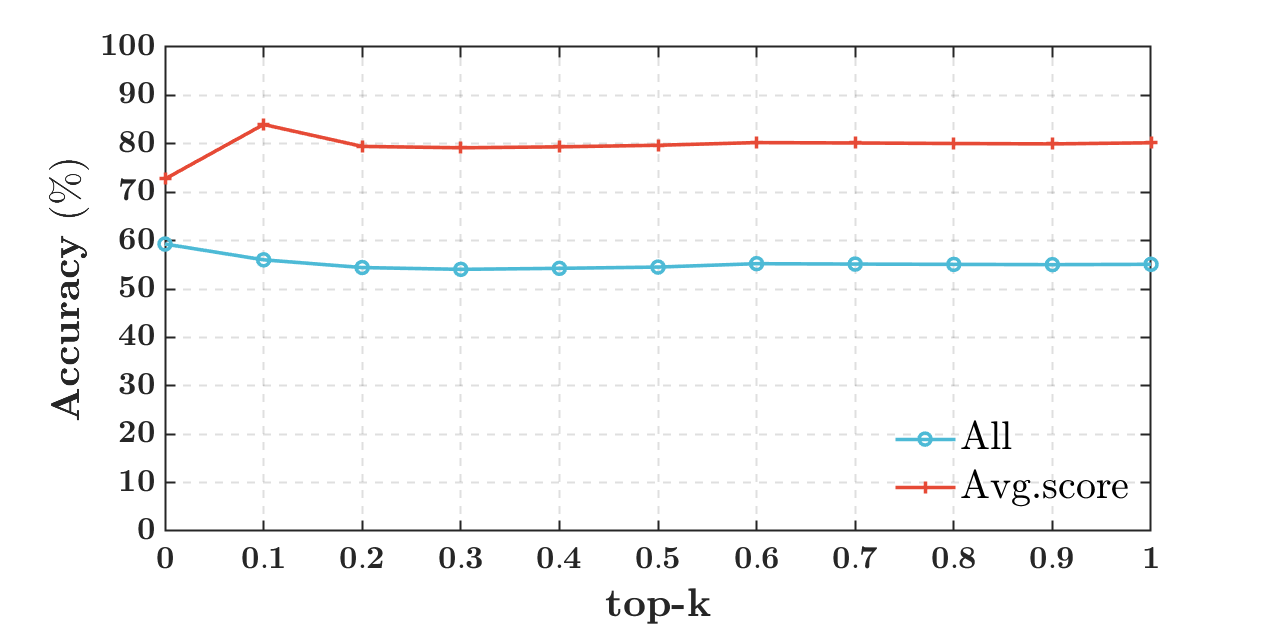}
    }
    \caption{Ablation study of the top-$k\%$ selection on ImageNet-1K in continual unlearning.}
    \label{fig:top-k}
    \vspace{-1em}
\end{figure}

\section{Conclusion}
\label{sec:conclusion}

In this paper, we study continual machine unlearning for vision-language models, where sequential forget requests introduce unique challenges in effectiveness, fidelity, and persistence. Unlike single-shot unlearning, this setting requires the model to remove newly specified target knowledge while preserving retained utility and preventing previously forgotten knowledge from re-emerging in later steps. To address these challenges, we propose \projectname, a framework based on Conflict-averse Task Arithmetic. By representing each forget set as an unlearning task vector and performing conflict-aware aggregation over historical vectors, \projectname{} suppresses inconsistent update directions and produces stable unlearning updates. Experiments on CLIP-based models demonstrate that \projectname{} achieves a strong balance among forgetting effectiveness, model fidelity, and persistence against knowledge re-emergence. These results highlight the importance of explicitly resolving inter-step conflicts for reliable continual unlearning in vision-language models.

\bibliographystyle{IEEEtran}
\bibliography{neurips_2026}

\newpage
\appendix
\section{Theoretical Analysis}
\label{sec:theory}
As a simple analysis, we formulate the regret bound of the proposed CATA in the convex setting. Firstly, we introduce convexity, smoothness, and Lipschitz assumptions. 

\begin{assumption}[Convexity]
Each \(\ell_t: \mathbb{R}^d \to \mathbb{R}\) is convex.
\end{assumption}

\begin{assumption}[\(\mu\)-smoothness]
Each \(\ell_t\) is \(\mu\)-smooth, \textit{i.e.}, its gradient is \(\mu\)-Lipschitz:
\[
\|\nabla \ell_t(\theta) - \nabla \ell_t(\theta')\| \le \mu \|\theta - \theta'\|, \qquad \forall \theta, \theta'.
\]
\end{assumption}

\begin{assumption}[\(L\)-Lipschitzness]
Each \(\ell_t\) is \(L\)-Lipschitz, i.e.,
\[
|\ell_t(\theta) - \ell_t(\theta')| \le L \|\theta - \theta'\|, \qquad \forall \theta, \theta'.
\]
Equivalently, \(\|\nabla \ell_t(\theta)\| \le L\) whenever differentiable.
\end{assumption}

Then, we assume that the task vector is uniformly similar to the genuine gradient.
\begin{assumption}[Gradient approximation error]
For each step \(t\), the sparse task vector satisfies
\[
\|\hat{\tau}^{(t)} + \nabla \ell_t(\theta^{(0)})\| \le \delta,
\]
where \(\delta \ge 0\) is a constant that depends on the fine-tuning procedure, sparsification ratio, and the Lipschitz constant of the loss.
\end{assumption}

Moreover, we upper bound the accumulated discrepancy between $G_{t}$ and the genuine gradient. 
\begin{lemma}[Cumulative gradient approximation error]
Let \(G_t = \sum_{j=1}^t \hat{\tau}^{(j)}\) and let \(\nabla \ell_j(\theta_j)\) be the true gradient at the point \(\theta_j\) chosen by CATA. Then
\[
\biggl\| G_t - \sum_{j=1}^t \nabla \ell_j(\theta_j) \biggr\| \le \mu R t + t\delta.
\]
\end{lemma}
\begin{proof}
By the triangle inequality,
\[
\biggl\| \sum_{j=1}^t \hat{\tau}^{(j)} - \sum_{j=1}^t \nabla \ell_j(\theta_j) \biggr\|
\le \sum_{j=1}^t \|\hat{\tau}^{(j)} + \nabla \ell_j(\theta^{(0)})\| + \sum_{j=1}^t \|\nabla \ell_j(\theta^{(0)}) - \nabla \ell_j(\theta_j)\|.
\]
The first sum is bounded by \(t\delta\) by Assumption 4. \end{proof}

Finally, we derive the Regret upper bound of CATA by 1) upper-bounding the accumulated discrepancy between $G_{t}$ and the genuine gradient (Lemma 1), 2) upper-bounding the regret bound induced by the genuine gradient.  
\begin{theorem}[Regret bound of CATA]
Under Assumptions 1–4, if the step-size is chosen as \(\eta = \frac{R}{(L+\delta)\sqrt{T}}\), then the cumulative regret of CATA satisfies
\[
\mathrm{Regret}_T = \sum_{t=1}^T \ell_t(\theta_t) - \min_{\theta} \sum_{t=1}^T \ell_t(\theta) \le R(L+\delta)\sqrt{T} + 2\mu R^2 T + 2R\delta T.
\]
\end{theorem}

\begin{proof}
We decompose the regret into three parts:
\[
\mathrm{Regret}_T = \underbrace{\sum_{t=1}^T \bigl(\ell_t(\theta_t) - \ell_t(\theta^*)\bigr)}_{\text{true regret}}.
\]

By convexity (Assumption 1),
\[
\ell_t(\theta_t) - \ell_t(\theta^*) \le \langle \nabla \ell_t(\theta_t), \theta_t - \theta^* \rangle.
\]

Besides, 
\[
\langle \nabla \ell_t(\theta_t), \theta_t - \theta^* \rangle = \langle \hat{\tau}^{(t)}, \theta_t - \theta^* \rangle + \langle \nabla \ell_t(\theta_t) - \hat{\tau}^{(t)}, \theta_t - \theta^* \rangle.
\]
The first term is the linear regret with respect to the approximate gradients \(\hat{\tau}^{(t)}\). For the update \(\theta_t = \theta^{(0)} - \eta G_{t-1}\) with \(G_{t-1}=\sum_{j=1}^{t-1}\hat{\tau}^{(j)}\), standard analysis gives
\[
\sum_{t=1}^T \langle \hat{\tau}^{(t)}, \theta_t - \theta^* \rangle \le \frac{\|\theta^* - \theta^{(0)}\|^2}{2\eta} + \frac{\eta}{2} \sum_{t=1}^T \|\hat{\tau}^{(t)}\|^2.
\]
Using \(\|\theta^* - \theta^{(0)}\|\le R\) and \(\|\hat{\tau}^{(t)}\| \le L+\delta\) (by Assumptions 3 and 4), and choosing \(\eta = R/((L+\delta)\sqrt{T})\), we obtain
\[
\sum_{t=1}^T \langle \hat{\tau}^{(t)}, \theta_t - \theta^* \rangle \le \frac{R^2}{2\eta} + \frac{\eta}{2} T (L+\delta)^2 = R (L+\delta) \sqrt{T}.
\]

Using Cauchy–Schwarz and the bound \(\|\theta_t - \theta^*\| \le 2R\),
\[
\sum_{t=1}^T \langle \nabla \ell_t(\theta_t) - \hat{\tau}^{(t)}, \theta_t - \theta^* \rangle
\le \sum_{t=1}^T \|\nabla \ell_t(\theta_t) - \hat{\tau}^{(t)}\| \cdot \|\theta_t - \theta^*\|
\le 2R \sum_{t=1}^T \|\nabla \ell_t(\theta_t) - \hat{\tau}^{(t)}\|.
\]
Now,
\[
\|\nabla \ell_t(\theta_t) - \hat{\tau}^{(t)}\|
\le \|\nabla \ell_t(\theta_t) - \nabla \ell_t(\theta^{(0)})\| + \|\nabla \ell_t(\theta^{(0)}) + \hat{\tau}^{(t)}\|
\le \mu R + \delta.
\]
Hence
\[
\sum_{t=1}^T \langle \nabla \ell_t(\theta_t) - \hat{\tau}^{(t)}, \theta_t - \theta^* \rangle \le 2R (\mu R + \delta) T.
\]

Putting everything above together,
\[
\mathrm{Regret}_T \le R(L+\delta)\sqrt{T} + 2R\mu R T + 2R\delta T = R(L+\delta)\sqrt{T} + 2\mu R^2 T + 2R\delta T.
\]
\end{proof}
\section{Additional Implementation Details}
\label{sec:supp_baseline_details}

All experiments are conducted on a single NVIDIA RTX 4090 GPU with 48GB memory, using Python~3.11.14 and PyTorch~2.9.1. All CLIP backbones are initialized from the official OpenAI released pretrained weights\footnote{\url{https://github.com/openai/CLIP}}. In the main continual unlearning experiments on ImageNet-1K, we sequentially remove five target classes: \textit{tailed frog}, \textit{diamondback}, \textit{barn spider}, \textit{briard}, and \textit{dungeness crab}. For the implementation details of baselines, we follow the official implementations or commonly adopted settings from prior work as described below.

\begin{itemize}
    \item \textbf{FT}~\cite{warnecke2023machine} updates the image encoder using only the retain set $\mathcal{D}_r$, while keeping the text encoder frozen, and optimizes the standard CLIP contrastive loss. For ImageNet, we subsample approximately 50k retain samples per step to reduce computation. We train for 2 epochs per step using Adam with a learning rate of $1\times10^{-6}$ and a batch size of 128.

    \item \textbf{GA}~\cite{thudi2022unrolling} performs gradient ascent on the current forget set $\mathcal{D}_u^{(t)}$ to maximize the CLIP loss and disrupt image--text alignment for target classes. Only the image encoder is updated, with gradient norm clipping applied for stability. We use Adam with a learning rate of $1\times10^{-6}$, a batch size of 128, and train for 2 epochs per step with gradient clipping set to 1.0.

    \item \textbf{Fisher}~\cite{golatkar2020eternal} estimates the Fisher Information Matrix on the forgetting dataset and perturbs the model parameters by adding Gaussian noise whose variance is inversely proportional to the Fisher information, thereby reducing the model's sensitivity to the target data. In the continual setting, we apply this procedure sequentially at each step based on $\mathcal{D}_u^{(t)}$.

    \item \textbf{LIP}~\cite{foster2024information} freezes the image encoder and updates the text projection matrix via a LoRA adapter, optimizing a combination of forget loss, retain loss, and a regularization term. To reduce computation on ImageNet, up to 512 retain classes are sampled per step. After optimization, the LoRA weights are merged back into the model. We use Adam with a learning rate of 0.01, LoRA rank 5, and run 2,000 iterations per step with $\lambda_1=0.3$, dynamically adjusted $\lambda_2$, and $\lambda_3=1.0$.

    \item \textbf{EMMN}~\cite{chundawat2023zero} adopts a teacher--student framework with a learnable pseudo-image generator. The teacher model is fixed, while the student is updated using pseudo-images filtered by confidence, and training alternates between generator and student updates in a 1:5 ratio. We use Adam for both components, with a student learning rate of $1\times10^{-6}$, a generator learning rate of $1\times10^{-5}$, batch size 128, temperature $T=1.0$, and confidence threshold $\delta=0.5$.

    \item \textbf{CLIP-LIP}~\cite{kravets2025zero} extends LIP to the CLIP framework by applying low-rank adaptation to the text projection while preserving the pretrained image encoder. Similar to LIP, it updates the text projection using a low-rank LoRA module to suppress representations of the forgetting classes while maintaining alignment for retained classes. In the continual setting, we apply CLIP-LIP sequentially, where each step updates the model using the cumulative forgetting set. We follow the original implementation and use Adam with a learning rate of 0.01, LoRA rank 5, and 2,000 iterations per step.

    \item \textbf{TIFS}~\cite{cai2025targeted} performs targeted forgetting by suppressing the model's responses on the forgetting dataset while preserving performance on retained data. Specifically, it minimizes prediction confidence on target classes while maintaining alignment on non-target samples. In the continual setting, we apply TIFS sequentially, where at step $t$ the model is updated using the current forgetting set $\mathcal{D}_u^{(t)}$ together with retained data. We follow the original implementation and update only the image encoder using Adam with a learning rate of $1\times10^{-6}$, batch size 128, and train for 2 epochs per step.
\end{itemize}

\section{Additional Experimental Results}

\partitle{Additional comparison experiments in single-shot unlearning}
Table~\ref{tab:cifar10} reports additional single-shot unlearning results on CIFAR-10. Overall, \projectname{} achieves the best \textit{Avg. Score} on both RN50 and RN101, showing strong utility preservation after unlearning. Compared with FT, GA, Fisher, and LIP, which substantially degrade retained and transfer performance, \projectname{} maintains much higher accuracy on retained classes and downstream datasets. Although CLIP-LIP and TIFS preserve model utility in some settings, they often retain higher target accuracy or show weaker overall balance. These results further confirm that \projectname{} achieves effective target removal while preserving model fidelity in single-shot unlearning. These observations are consistent with the main experimental results on ImageNet-1K in Table~\ref{tab:single-shot imagenet}, further validating the robustness of \projectname{} across different datasets and backbones.

\begin{table*}[htbp]
\centering
\caption{Performance comparison on CIFAR-10 in single-shot unlearning.}
\resizebox{\textwidth}{!}{
\begin{tabular}{cccccccccccccc}
\toprule
\multirow{2}{*}{Backbone} & \multirow{2}{*}{Method} & \multicolumn{3}{c}{CIFAR-10} & \multirow{2}{*}{Food$\uparrow$} & \multirow{2}{*}{STL$\uparrow$} & \multirow{2}{*}{ObjectNet$\uparrow$} & \multicolumn{2}{c}{ImageNet} & \multirow{2}{*}{Avg. Score$\uparrow$} \\ \cmidrule{3-5} \cmidrule{9-10}
& & Target$\downarrow$ & Retain$\uparrow$ & All$\uparrow$ & & & &  Target$\downarrow$ & All$\uparrow$  \\ 

\midrule

\multirow{9}{*}{RN50}
& Original & 54.10& 66.61 & 65.37 & 76.49 & 93.75 & 25.83 & 71.33 & 53.81 & -- \\  \cmidrule{2-11}
& FT  \cite{warnecke2023machine} & $22.40_{41.40}$  & $63.29_{95.02}$ & $58.20_{89.03}$ & $1.70_{2.22}$ & $45.20_{48.21}$ & $0.29_{1.12}$ & $0.10_{0.14}$ & $0.00_{0.00}$ & 36.79 \\
& GA \cite{thudi2022unrolling}  & $5.90_{10.91}$  & $27.01_{40.55}$ & $24.90_{38.09}$ & $1.82_{2.38}$ & $34.16_{36.44}$ & $0.88_{3.41}$ & $0.00_{0.00}$ & $0.40_{0.74}$ & 26.34 \\
& Fisher \cite{golatkar2020eternal} & $0.00_{0.00}$  & $12.39_{18.60}$ & $12.16_{18.60}$ & $1.22_{1.60}$ & $15.47_{16.50}$ & $0.10_{0.40}$ & $0.00_{0.00}$ & $0.11_{0.20}$ & 19.49 \\
& LIP \cite{foster2024information}  & $0.00_{0.00}$  & $14.85_{22.30}$ & $14.58_{22.30}$ & $1.61_{2.10}$ & $15.19_{16.20}$ & $0.08_{0.30}$ & $0.14_{0.20}$ & $0.16_{0.30}$ & 20.46 \\
& EMMN \cite{chundawat2023zero}  & $0.30_{0.55}$  & $11.34_{17.02}$ & $10.24_{15.66}$ & $49.70_{64.98}$ & $46.86_{49.98}$ & $10.21_{39.53}$ & $48.00_{67.29}$ & $36.32_{67.50}$ & 52.68 \\
& CLIP-LIP \cite{kravets2025zero} & $67.80_{100.00}$  & $68.44_{100.00}$ & $68.34_{100.00}$ & $76.78_{100.00}$ & $93.66_{99.90}$ & $25.91_{100.00}$ & $57.33_{80.37}$ & $53.85_{100.00}$ & 81.87 \\
& TIFS \cite{zhang2025targeted}    & $5.03_{9.30}$  & $61.61_{92.50}$ & $60.47_{92.50}$ & $65.17_{85.20}$ & $87.38_{93.20}$ & $21.70_{84.00}$ & $66.27_{92.90}$ & $45.09_{83.80}$ & 89.35 \\ \cmidrule{2-11}
& \projectname \ (ours) & $3.30_{6.10}$  & $69.46_{100.00}$ & $62.84_{96.13}$ & $74.43_{97.31}$ & $88.84_{94.76}$ & $24.51_{94.89}$ & $31.55_{44.24}$ & $50.25_{93.38}$ & $\textbf{96.52}$ \\

\midrule
\multirow{8}{*}{RN101}
& Original & 70.60 & 75.39 & 74.91 & 81.10 & 96.45 & 29.16 & 66.67 & 55.25 & --\\ \cmidrule{2-11}
& FT  \cite{warnecke2023machine}  
& $14.90_{21.10}$  & $63.84_{84.68}$ & $57.89_{77.28}$ & $1.70_{2.10}$ & $66.90_{69.36}$ & $1.87_{6.41}$ & $0.00_{0.00}$ & $1.29_{2.33}$ & 40.13 \\

& GA  \cite{thudi2022unrolling}      
& $4.50_{6.37}$  & $25.86_{34.30}$ & $23.72_{31.66}$ & $3.47_{4.28}$ & $56.33_{58.40}$ & $1.97_{6.76}$ & $0.00_{0.00}$ & $1.64_{2.97}$ & 29.00 \\

& Fisher \cite{golatkar2020eternal}    
& $0.00_{0.00}$  & $15.08_{20.00}$ & $14.98_{20.00}$ & $1.38_{1.70}$ & $17.65_{18.30}$ & $0.12_{0.40}$ & $0.13_{0.20}$ & $0.17_{0.30}$ & 20.11 \\

& LIP  \cite{foster2024information}     
& $0.00_{0.00}$  & $16.13_{21.40}$ & $16.03_{21.40}$ & $2.27_{2.80}$ & $16.40_{17.00}$ & $0.09_{0.30}$ & $0.00_{0.00}$ & $0.28_{0.50}$ & 20.43 \\

& EMMN \cite{chundawat2023zero}    
& $78.30_{100.00}$  & $6.47_{8.58}$ & $13.65_{18.22}$ & $45.29_{55.84}$ & $42.85_{44.43}$ & $10.71_{36.73}$ & $62.67_{94.00}$ & $36.54_{66.14}$ & 39.13 \\

& CLIP-LIP \cite{kravets2025zero} 
& $3.60_{5.10}$  & $75.67_{100.00}$ & $68.50_{91.44}$ & $79.11_{97.55}$ & $90.88_{94.22}$ & $26.83_{92.01}$ & $21.33_{31.99}$ & $55.25_{100.00}$ & 87.76 \\
 
& TIFS \cite{zhang2025targeted} 
& $8.97_{12.70}$  & $73.96_{98.10}$ & $73.49_{98.10}$ & $69.10_{85.20}$ & $91.53_{94.90}$ & $24.49_{84.00}$ & $58.74_{88.10}$ & $36.52_{66.10}$ & 87.73 \\  \cmidrule{2-11}

& \projectname \ (ours)       
& $3.05_{4.31}$  & $69.41_{92.07}$ & $70.77_{94.47}$ & $78.97_{97.37}$ & $89.71_{93.01}$ & $28.27_{96.95}$ & $27.33_{40.99}$ & $51.02_{92.34}$ & $\textbf{87.86}$ \\

\bottomrule
\end{tabular}}

\label{tab:cifar10}
\end{table*}

\partitle{Additional comparison experiments in continual unlearning}
Table~\ref{tab:supp_stepwise_cifar100_all} reports additional step-wise continual unlearning results on CIFAR-100. The observations are consistent with the main results on ImageNet-1K. FT generally preserves high retain and downstream performance but fails to effectively forget the target classes. GA achieves stronger target suppression than FT, but its retained and overall performance degrades as the number of unlearning steps increases. LIP aggressively suppresses many target classes, but severely damages model utility across retained and transfer datasets. In contrast, \projectname{} consistently reduces the accuracy of the forgotten classes while maintaining stronger downstream performance, demonstrating a better balance among effectiveness, fidelity, and irreversibility across different datasets and backbones.

\begin{table*}[htbp]
\centering
\caption{Performance comparison on CIFAR-100 in continual unlearning. Underlined values indicate the class forgotten at the corresponding step.}
\label{tab:supp_stepwise_cifar100_all}
\scriptsize
\setlength{\tabcolsep}{3.2pt}
\resizebox{\textwidth}{!}{
\begin{tabular}{ccc|ccccc|cccccccc}
\toprule[1pt]
\multirow{2}{*}{Backbone} 
& \multirow{2}{*}{Method} 
& \multirow{2}{*}{Step} 
& \multicolumn{5}{c|}{Target$\downarrow$} 
& \multirow{2}{*}{Retain$\uparrow$} 
& \multirow{2}{*}{All$\uparrow$} 
& \multirow{2}{*}{Food$\uparrow$} 
& \multirow{2}{*}{STL$\uparrow$} 
& \multirow{2}{*}{ObjectNet$\uparrow$} 
& \multirow{2}{*}{CIFAR-10$\uparrow$} 
& \multirow{2}{*}{Avg. $\Delta\downarrow$} 
& \multirow{2}{*}{Avg. Score$\uparrow$} \\
& & 
& bed & crab & house & rose & wardro 
& & & & & & & & \\
\midrule[1pt]

\multirow{21}{*}{ViT-B/32}
& Original & Step 0 
& $66.00$ & $61.00$ & $72.00$ & $61.00$ & $80.00$ 
& $61.20$ & $61.54$ & $82.05$ & $97.36$ & $30.27$ & $88.79$ & -- & -- \\
\cmidrule{2-16}

& \multirow{5}{*}{GA \cite{thudi2022unrolling}}
& Step 1 & $\underline{11.00}$ & $38.00$ & $66.00$ & $56.00$ & $65.00$ & $57.77$ & $57.24$ & $81.79$ & $97.41$ & $29.85$ & $87.28$  & -- & -- \\
& & Step 2 & $14.00$ & $\underline{22.00}$ & $59.00$ & $48.00$ & $79.00$ & $53.15$ & $52.71$ & $81.90$ & $97.41$ & $29.85$ & $84.10$  & -- & -- \\
& & Step 3 & $15.00$ & $21.00$ & $\underline{40.00}$ & $44.00$ & $78.00$ & $52.13$ & $51.50$ & $81.66$ & $97.33$ & $29.66$ & $81.63$  & -- & -- \\
& & Step 4 & $8.00$ & $11.00$ & $23.00$ & $\underline{19.00}$ & $63.00$ & $39.12$ & $38.40$ & $81.03$ & $97.03$ & $29.23$ & $69.93$  & -- & -- \\
& & Step 5 & $15.00$ & $17.00$ & $38.00$ & $23.00$ & $\underline{52.00}$ & $43.58$ & $42.85$ & $81.41$ & $97.21$ & $29.52$ & $75.68$  & $3.00$ & $74.23$ \\
\cmidrule{2-16}

& \multirow{5}{*}{FT \cite{warnecke2023machine}}
& Step 1 & $\underline{70.00}$ & $74.00$ & $85.00$ & $90.00$ & $95.00$ & $83.17$ & $83.15$ & $79.32$ & $97.38$ & $29.53$ & $91.66$  & -- & -- \\
& & Step 2 & $65.00$ & $\underline{60.00}$ & $82.00$ & $89.00$ & $96.00$ & $84.62$ & $84.31$ & $77.40$ & $96.81$ & $28.81$ & $90.37$  & -- & -- \\
& & Step 3 & $66.00$ & $53.00$ & $\underline{79.00}$ & $91.00$ & $96.00$ & $85.27$ & $84.86$ & $76.35$ & $96.50$ & $28.33$ & $89.49$  & -- & -- \\
& & Step 4 & $58.00$ & $51.00$ & $77.00$ & $\underline{87.00}$ & $96.00$ & $84.93$ & $84.37$ & $75.60$ & $96.28$ & $27.81$ & $88.47$  & -- & -- \\
& & Step 5 & $51.00$ & $48.00$ & $73.00$ & $73.00$ & $\underline{96.00}$ & $84.98$ & $84.14$ & $75.14$ & $95.74$ & $27.49$ & $87.28$  & $10.20$ & $59.78$ \\
\cmidrule{2-16}

& \multirow{5}{*}{LIP \cite{foster2024information}}
& Step 1 & $\underline{0.00}$ & $0.00$ & $0.00$ & $0.00$ & $0.00$ & $2.46$ & $2.34$ & $10.50$ & $34.60$ & $4.51$ & $10.89$  & -- & -- \\
& & Step 2 & $69.00$ & $\underline{54.00}$ & $0.00$ & $0.00$ & $0.00$ & $0.22$ & $1.42$ & $1.06$ & $9.71$ & $0.33$ & $8.88$  & -- & -- \\
& & Step 3 & $25.00$ & $23.00$ & $\underline{13.00}$ & $0.00$ & $0.00$ & $0.63$ & $1.19$ & $1.10$ & $7.24$ & $0.19$ & $9.30$  & -- & -- \\
& & Step 4 & $52.00$ & $10.00$ & $10.00$ & $\underline{0.00}$ & $0.00$ & $0.51$ & $1.21$ & $1.06$ & $8.14$ & $0.11$ & $9.40$  & -- & -- \\
& & Step 5 & $42.00$ & $1.00$ & $3.00$ & $0.00$ & $\underline{9.00}$ & $0.61$ & $1.10$ & $1.14$ & $9.85$ & $0.10$ & $9.51$  & $21.00$ & $40.42$ \\ \cmidrule{2-16}

& \multirow{5}{*}{\projectname\ (ours)}
& Step 1 & $\underline{0.00}$ & $45.00$ & $68.00$ & $44.00$ & $62.00$ & $55.69$ & $55.10$ & $81.67$ & $97.38$ & $29.95$ & $85.15$  & -- & -- \\
& & Step 2 & $2.00$ & $\underline{0.00}$ & $64.00$ & $38.00$ & $76.00$ & $45.38$ & $44.91$ & $81.68$ & $97.06$ & $29.46$ & $75.19$  & -- & -- \\
& & Step 3 & $1.00$ & $2.00$ & $\underline{7.00}$ & $38.00$ & $53.00$ & $47.43$ & $46.07$ & $82.02$ & $97.12$ & $29.90$ & $77.30$  & -- & -- \\
& & Step 4 & $1.00$ & $2.00$ & $9.00$ & $\underline{3.00}$ & $44.00$ & $46.61$ & $44.87$ & $81.74$ & $97.09$ & $29.61$ & $76.31$  & -- & -- \\
& & Step 5 & $1.00$ & $5.00$ & $8.00$ & $7.00$ & $\underline{5.00}$ & $50.39$ & $48.13$ & $81.72$ & $97.16$ & $29.78$ & $79.16$  & $\textbf{2.20}$ & $\textbf{91.72}$ \\

\midrule[1pt]

\multirow{21}{*}{ViT-L/14}
& Original & Step 0 
& $90.00$ & $82.00$ & $71.00$ & $91.00$ & $62.00$ 
& $74.88$ & $75.10$ & $92.05$ & $99.41$ & $51.86$ & $95.32$  & -- & -- \\
\cmidrule{2-16}

& \multirow{5}{*}{GA \cite{thudi2022unrolling}}
& Step 1 & $\underline{26.00}$ & $64.00$ & $32.00$ & $78.00$ & $40.00$ & $61.21$ & $60.55$ & $91.90$ & $99.40$ & $51.61$ & $93.42$  & -- & -- \\
& & Step 2 & $10.00$ & $\underline{9.00}$ & $7.00$ & $10.00$ & $15.00$ & $19.29$ & $18.84$ & $91.72$ & $99.26$ & $51.36$ & $71.29$  & -- & -- \\
& & Step 3 & $15.00$ & $20.00$ & $\underline{8.00}$ & $16.00$ & $14.00$ & $24.77$ & $24.26$ & $91.71$ & $99.24$ & $51.41$ & $71.01$  & -- & -- \\
& & Step 4 & $14.00$ & $16.00$ & $4.00$ & $\underline{11.00}$ & $3.00$ & $10.69$ & $10.64$ & $91.68$ & $99.19$ & $51.24$ & $57.93$  & -- & -- \\
& & Step 5 & $13.00$ & $7.00$ & $0.00$ & $5.00$ & $\underline{2.00}$ & $6.01$ & $5.98$ & $91.58$ & $99.15$ & $51.33$ & $31.49$  & $5.80$ & $74.14$ \\
\cmidrule{2-16}

& \multirow{5}{*}{FT \cite{warnecke2023machine}}
& Step 1 & $\underline{87.00}$ & $88.00$ & $89.00$ & $93.00$ & $96.00$ & $89.55$ & $89.60$ & $90.99$ & $98.94$ & $50.93$ & $95.78$  & -- & -- \\
& & Step 2 & $69.00$ & $\underline{80.00}$ & $90.00$ & $94.00$ & $98.00$ & $90.12$ & $89.92$ & $90.19$ & $98.84$ & $49.82$ & $95.31$  & -- & -- \\
& & Step 3 & $74.00$ & $75.00$ & $\underline{89.00}$ & $91.00$ & $98.00$ & $90.73$ & $90.46$ & $89.53$ & $98.05$ & $48.82$ & $94.05$  & -- & -- \\
& & Step 4 & $71.00$ & $69.00$ & $87.00$ & $\underline{88.00}$ & $97.00$ & $91.18$ & $90.74$ & $89.53$ & $97.90$ & $47.95$ & $95.09$  & -- & -- \\
& & Step 5 & $80.00$ & $61.00$ & $89.00$ & $76.00$ & $\underline{97.00}$ & $91.38$ & $90.84$ & $89.18$ & $98.06$ & $47.93$ & $93.66$ & $7.60$ & $54.60$ \\
\cmidrule{2-16}

& \multirow{5}{*}{LIP \cite{foster2024information}}
& Step 1 & $\underline{0.00}$ & $37.00$ & $0.00$ & $0.00$ & $0.00$ & $8.89$ & $8.80$ & $39.62$ & $67.38$ & $21.74$ & $59.95$  & -- & -- \\
& & Step 2 & $24.00$ & $\underline{94.00}$ & $0.00$ & $0.00$ & $1.00$ & $0.03$ & $1.22$ & $1.01$ & $18.15$ & $3.96$ & $12.13$  & -- & -- \\
& & Step 3 & $13.00$ & $93.00$ & $\underline{2.00}$ & $0.00$ & $1.00$ & $0.02$ & $1.11$ & $0.79$ & $11.18$ & $3.07$ & $12.14$  & -- & -- \\
& & Step 4 & $11.00$ & $93.00$ & $2.00$ & $\underline{0.00}$ & $0.00$ & $0.02$ & $1.08$ & $0.82$ & $12.00$ & $3.13$ & $11.90$  & -- & -- \\
& & Step 5 & $10.00$ & $84.00$ & $2.00$ & $0.00$ & $\underline{8.00}$ & $0.01$ & $1.03$ & $0.82$ & $13.35$ & $3.31$ & $11.98$  & $4.00$ & $36.85$ \\
\cmidrule{2-16}

& \multirow{5}{*}{\projectname\ (ours)}
& Step 1 & $\underline{0.00}$ & $77.00$ & $61.00$ & $81.00$ & $23.00$ & $65.39$ & $64.39$ & $92.12$ & $99.40$ & $51.76$ & $94.12$  & -- & -- \\
& & Step 2 & $26.00$ & $\underline{0.00}$ & $44.00$ & $36.00$ & $40.00$ & $49.46$ & $48.45$ & $91.85$ & $99.40$ & $51.40$ & $84.64$  & -- & -- \\
& & Step 3 & $24.00$ & $2.00$ & $\underline{4.00}$ & $43.00$ & $28.00$ & $50.74$ & $49.21$ & $91.99$ & $99.38$ & $51.34$ & $87.38$  & -- & -- \\
& & Step 4 & $25.00$ & $2.00$ & $6.00$ & $\underline{2.00}$ & $28.00$ & $46.12$ & $44.44$ & $91.79$ & $99.36$ & $51.26$ & $86.42$  & -- & -- \\
& & Step 5 & $4.00$ & $4.00$ & $5.00$ & $3.00$ & $\underline{1.00}$ & $50.55$ & $50.41$ & $91.79$ & $99.33$ & $50.59$ & $88.41$  & $\textbf{2.00}$ & $\textbf{91.21}$ \\

\bottomrule[1pt]
\end{tabular}}
\end{table*}

\partitle{Detailed results for scalability evaluation}
Table~\ref{tab:vitB_in10_lam09} provides the detailed numerical results corresponding to the scalability evaluation in the main paper, where the results are visualized in figure form. We evaluate \projectname{} on ImageNet-1K with ViT-B/32 by increasing the number of unlearning steps to 10, with each step forgetting one target class. The underlined value indicates the class forgotten at the corresponding step.
The results show that \projectname{} remains effective as the unlearning sequence becomes longer. Once a class is forgotten, its target accuracy generally stays close to zero in subsequent steps, indicating that the proposed method can suppress knowledge re-emergence over longer sequences. Meanwhile, the retain and overall accuracy remain relatively stable across 10 steps, further supporting the scalability of \projectname{} for continual unlearning.

\begin{table*}[htbp]
\centering
\caption{Detailed numerical results for the scalability evaluation on ImageNet-1K. Each step forgets one target class, and underlined values indicate the class forgotten at the corresponding step.}
\resizebox{\textwidth}{!}{
\begin{tabular}{ccccccccccccc}
\toprule[1pt]
\multirow{2}{*}{} & \multicolumn{10}{c}{Target$\downarrow$} & \multirow{2}{*}{Retain$\uparrow$} & \multirow{2}{*}{All$\uparrow$} \\
& tailed frog & diamondback & barn spider & briard & dungeness crab & bib & loupe & radio telescope & dock & streetcar & & \\
\midrule[1pt]
Step 0 & $64.00$ & $24.00$ & $32.00$ & $22.00$ & $90.00$ & $46.00$ & $30.00$ & $50.00$ & $22.00$ & $52.00$ & $59.46$ & $59.29$ \\ \cmidrule{1-13}
Step 1 & \multicolumn{1}{c|}{$\underline{0.00}$} & $10.00$ & $14.00$ & $34.00$ & $66.00$ & $48.00$ & $34.00$ & $52.00$ & $24.00$ & $60.00$ & $53.94$ & $53.74$ \\ \cmidrule{3-3}
Step 2 & $0.00$ & \multicolumn{1}{c|}{$\underline{0.00}$} & $18.00$ & $38.00$ & $62.00$ & $44.00$ & $34.00$ & $48.00$ & $20.00$ & $58.00$ & $53.88$ & $53.67$ \\ \cmidrule{4-4}
Step 3 & $0.00$ & $0.00$ & \multicolumn{1}{c|}{$\underline{0.00}$} & $24.00$ & $50.00$ & $44.00$ & $28.00$ & $36.00$ & $22.00$ & $54.00$ & $53.78$ & $53.50$ \\ \cmidrule{5-5}
Step 4 & $0.00$ & $0.00$ & $0.00$ & \multicolumn{1}{c|}{$\underline{0.00}$} & $58.00$ & $44.00$ & $26.00$ & $34.00$ & $28.00$ & $54.00$ & $53.14$ & $52.86$ \\ \cmidrule{6-6}
Step 5 & $0.00$ & $0.00$ & $0.00$ & $0.00$ & \multicolumn{1}{c|}{$\underline{2.00}$} & $42.00$ & $22.00$ & $30.00$ & $28.00$ & $56.00$ & $53.44$ & $53.09$ \\ \cmidrule{7-7}
Step 6 & $0.00$ & $0.00$ & $0.00$ & $0.00$ & $2.00$ & \multicolumn{1}{c|}{$\underline{0.00}$} & $16.00$ & $34.00$ & $28.00$ & $56.00$ & $54.42$ & $54.01$ \\ \cmidrule{8-8}
Step 7 & $2.00$ & $0.00$ & $0.00$ & $0.00$ & $2.00$ & $0.00$ & \multicolumn{1}{c|}{$\underline{2.00}$} & $46.00$ & $26.00$ & $56.00$ & $54.87$ & $54.45$ \\ \cmidrule{9-9}
Step 8 & $4.00$ & $0.00$ & $0.00$ & $0.00$ & $2.00$ & $0.00$ & $2.00$ & \multicolumn{1}{c|}{$\underline{2.00}$} & $26.00$ & $60.00$ & $54.93$ & $54.48$ \\ \cmidrule{10-10}
Step 9 & $4.00$ & $0.00$ & $0.00$ & $0.00$ & $6.00$ & $2.00$ & $2.00$ & $2.00$ & \multicolumn{1}{c|}{$\underline{4.00}$} & $52.00$ & $55.27$ & $54.79$ \\ \cmidrule{11-11}
Step 10& $4.00$ & $0.00$ & $0.00$ & $0.00$ & $6.00$ & $2.00$ & $2.00$ & $4.00$ & $4.00$ & \multicolumn{1}{c|}{$\underline{4.00}$} & $55.16$ & $54.64$ \\
\bottomrule[1pt]
\end{tabular}}
\label{tab:vitB_in10_lam09}
\end{table*}

\section{Naive averaging aggregation}
Naive averaging aggregation is used as an ablation variant of \projectname{} by replacing the proposed conflict-averse aggregation with direct averaging over all accumulated task vectors. Given the task vectors $\{\tau^{(i)}\}_{i=1}^{N}$ computed from the fixed original CLIP anchor, the averaged task vector is:
\begin{equation}
\tau_{\mathrm{avg}} = \frac{1}{N}\sum_{i=1}^{N}\tau^{(i)}.
\end{equation}
The unlearned model is then obtained as:
\begin{equation}
\theta_{\mathrm{u}}^{(t)} = \theta^{(0)} + \lambda \tau_{\mathrm{avg}}^{(t)}.
\end{equation}
All other settings, including task-vector construction, top-$k\%$ trimming, scaling factor, and evaluation protocol, are kept the same as \projectname{}.

\partitle{Additional ablation results for aggregation strategies}
Table~\ref{tab:aggregation_stepwise} provides additional step-wise results for the aggregation ablation on CIFAR-100. Naive average aggregation directly averages historical task vectors without considering sign conflicts across unlearning steps. As a result, several previously forgotten classes show increased target accuracy in later steps, indicating potential knowledge re-emergence caused by conflicting updates.
In contrast, \projectname{} consistently keeps the accuracy of forgotten classes close to zero across subsequent steps and achieves a lower Avg. $\Delta$ with a higher Avg. Score. These results further support the observations in the main paper: directly averaging task vectors is insufficient for stable continual unlearning, while conflict-averse aggregation better balances effectiveness, fidelity, and persistence.

\begin{table*}[htbp]
\centering
\caption{Ablation study of aggregation strategies on Imagenet in continual unlearning. Underlined values indicate the class forgotten at the corresponding step.}
\label{tab:aggregation_stepwise}
\scriptsize
\setlength{\tabcolsep}{3.5pt}
\resizebox{\textwidth}{!}{
\begin{tabular}{cc|ccccc|cccccccc}
\toprule[1pt]
\multirow{2}{*}{Aggregation}
& \multirow{2}{*}{Step} 
& \multicolumn{5}{c|}{Target$\downarrow$} 
& \multirow{2}{*}{Retain$\uparrow$} 
& \multirow{2}{*}{All$\uparrow$} 
& \multirow{2}{*}{Food$\uparrow$} 
& \multirow{2}{*}{STL$\uparrow$} 
& \multirow{2}{*}{ObjectNet$\uparrow$} 
& \multirow{2}{*}{CIFAR-10$\uparrow$} 
& \multirow{2}{*}{Avg. $\Delta\downarrow$} 
& \multirow{2}{*}{Avg. Score$\uparrow$} \\
& 
& Class 1 
& Class 2  
& Class 3 
& Class 4
& Class 5 
& & & & & & & & \\
\midrule[1pt]

Original & Step 0 
& $38.00$ & $46.00$ & $34.00$ & $52.00$ & $88.00$ 
& $71.72$ & $71.62$ & $92.05$ & $99.41$ & $51.86$ & $95.32$ 
& -- & -- \\
\cmidrule{1-15}

\multirow{5}{*}{Naive Avg.}
& Step 1 & $\underline{0.00}$ & $64.00$ & $36.00$ & $46.00$ & $88.00$ 
& $70.93$ & $70.81$ & $91.348$ & $99.42$ & $51.02$ & $90.62$ 
& -- & -- \\
& Step 2 & $0.00$ & $\underline{0.00}$ & $34.00$ & $50.00$ & $90.00$ 
& $71.23$ & $71.05$ & $91.60$ & $99.40$ & $51.43$ & $93.24$ 
& -- & -- \\
& Step 3 & $0.00$ & $0.00$ & $\underline{2.00}$ & $52.00$ & $90.00$ 
& $71.36$ & $71.15$ & $91.62$ & $99.42$ & $51.54$ & $93.95$ 
& -- & -- \\
& Step 4 & $0.00$ & $4.00$ & $4.00$ & $\underline{16.00}$ & $90.00$ 
& $71.19$ & $70.95$ & $91.62$ & $99.38$ & $51.57$ & $94.58$ 
& -- & -- \\
& Step 5 & $2.00$ & $6.00$ & $4.00$ & $20.00$ & $\underline{46.00}$ 
& $71.33$ & $71.05$ & $91.78$ & $99.41$ & $50.61$ & $94.49$ 
& $2.80$ & $88.57$ \\
\cmidrule{1-15}

\multirow{5}{*}{Ours}
& Step 1 & $\underline{0.00}$ & $66.00$ & $36.00$ & $46.00$ & $88.00$ 
& $70.66$ & $70.55$ & $91.22$ & $99.41$ & $51.19$ & $89.91$ & -- & -- \\
& Step 2 & $0.00$ & $\underline{0.00}$ & $40.00$ & $50.00$ & $90.00$ 
& $70.69$ & $70.52$ & $91.23$ & $99.36$ & $50.42$ & $90.89$ & -- & -- \\
& Step 3 & $0.00$ & $0.00$ & $\underline{0.00}$ & $52.00$ & $86.00$ 
& $70.71$ & $70.49$ & $91.13$ & $99.40$ & $50.59$ & $92.37$ & -- & -- \\
& Step 4 & $0.00$ & $0.00$ & $0.00$ & $\underline{0.00}$ & $90.00$ 
& $69.92$ & $69.66$ & $90.99$ & $99.22$ & $50.17$ & $93.65$ & -- & -- \\
& Step 5 & $0.00$ & $0.00$ & $0.00$ & $8.00$ & $\underline{10.00}$ 
& $70.22$ & $69.89$ & $91.11$ & $99.31$ & $50.15$ & $93.26$ & ${\textbf{1.60}}$ & ${\textbf{96.56}}$ \\

\bottomrule[1pt]
\end{tabular}}
\end{table*}

\section{Broader Impact and Limitations}

\subsection{Broader Impact}
\label{subsec:broader impact}
This work studies continual machine unlearning for vision-language models, aiming to support responsible deployment of large-scale multimodal systems. By enabling models to process sequential removal requests, \projectname{} can help address practical concerns related to privacy protection, copyright compliance, and the removal of undesirable or sensitive content. Compared with retraining from scratch, the proposed method provides a more efficient alternative, which may reduce computational cost and make unlearning more accessible in real-world deployments. At the same time, machine unlearning should be applied with care. Improved unlearning techniques may be used to remove harmful or unauthorized knowledge, but they may also be misused to selectively erase information for improper purposes, such as hiding model behavior or removing evidence of biased training data. Therefore, continual unlearning should be accompanied by transparent evaluation protocols and appropriate auditing mechanisms to ensure that deletion requests are handled faithfully and responsibly.

\subsection{Limitations}
\label{subsec:limitations}
Although \projectname{} shows strong performance under both single-shot and continual unlearning settings, several limitations remain. First, our experiments are mainly conducted on CLIP-based models and classification-oriented benchmarks. Extending the method to broader VLM architectures, such as generative multimodal models, remains an important direction. Second, our current setting focuses on class-level forgetting, where each unlearning step removes one target class. More fine-grained scenarios, such as instance-level, concept-level, or text-only forgetting, may introduce additional challenges. Third, the proposed method relies on task vectors obtained by fine-tuning on forget sets. Although the top-$k\%$ trimming strategy reduces storage overhead, maintaining task vectors for long unlearning sequences may still require additional memory. Finally, while our conflict-averse aggregation mitigates knowledge re-emergence, it does not provide formal deletion guarantees. Developing verifiable and theoretically grounded continual unlearning methods for VLMs remains an important future direction.


\end{document}